\documentclass{article}

\usepackage{times}
\usepackage{graphicx} 
\usepackage{subfigure} 

\usepackage{natbib}

\usepackage{algorithm}
\usepackage{algorithmic}

\usepackage[nohyperref,accepted]{icml2015_mod} 

\usepackage{hyperref}
\usepackage{url}
\usepackage{amsmath}
\usepackage{amsthm}
\usepackage{amssymb}
\usepackage{color}
\usepackage{tikz}
\usepackage[latin1]{inputenc}
\usepackage{algorithm}
\usepackage{algorithmic}

\usepackage{xr}
\newtheorem{Theorem}{Theorem}
\newtheorem{Lemma}{Lemma}
\newtheorem{Proposition}{Proposition}

\newtheorem*{Assumptions}{Assumptions}

\newcommand{\R}{{\mathbb R}}

\newcommand{\C}{{\mathbb C}}
\newcommand{\Z}{{\mathbb Z}}

\newcommand{\E}{\mathbb{E}}

\newcommand\ind{\protect\mathpalette{\protect\independenT}{\perp}}
\def\independenT#1#2{\mathrel{\rlap{$#1#2$}\mkern2mu{#1#2}}}

\newcommand{\td}{\mathrm{d}}

\newcommand{\I}{\mathrm{I}}
\newcommand{\diag}{\mathrm{diag}}

\newcommand{\cov}{\mathrm{Cov}}

\newcommand{\stress}{\textbf}
\newcommand{\defi}{\emph}

\newcommand{\const}{\mathrm{const}}
\newcommand{\vect}{\mathrm{vec}}

\newcommand{\adj}{\mathrm{adj}}
\newcommand{\1}{\mathrm{1}}

\usepackage{pgfplots}
\pgfplotsset{compat=newest}
\pgfplotsset{plot coordinates/math parser=false}       
\usetikzlibrary{arrows,shapes,plotmarks}

\newif\iffinal 

\iffinal
 
\else
 
\fi

\newlength\figheight
\newlength\figwidth

\tikzset{>=stealth'}
\tikzstyle{graphnode} = [circle,draw=black,minimum size=22pt,text centered,text width=22pt,inner sep=0pt] 
\tikzstyle{var}   =[graphnode,fill=white]
\tikzstyle{obs}   =[graphnode,fill=white]
\tikzstyle{inv}   =[rectangle,fill=none,draw=none]
\tikzstyle{hid}   =[graphnode,fill=gray,draw=gray,text=white]
\tikzstyle{edge}  =[draw=white,double=black,thick,-]

\newtheorem*{p1}{Proposition 1}
\newtheorem*{p2}{Proposition 2}

\icmltitlerunning{Causal Inference by Identification of Vector Autoregressive Processes with Hidden Components}

\begin{document} 

\twocolumn[
\icmltitle{Causal Inference by Identification of Vector Autoregressive Processes with Hidden Components}

\icmlauthor{Philipp Geiger\textsuperscript{a}}{pgeiger@tuebingen.mpg.de}
\icmlauthor{Kun Zhang\textsuperscript{a,b}}{kzhang@tuebingen.mpg.de}
\icmlauthor{Mingming Gong\textsuperscript{c}}{gongmingnju@gmail.com}
\icmlauthor{Dominik Janzing\textsuperscript{a}}{janzing@tuebingen.mpg.de}
\icmlauthor{Bernhard Sch\"olkopf\textsuperscript{a}}{bs@tuebingen.mpg.de}
\icmladdress{\textsuperscript{a}Empirical Inference Department, Max Planck Institute for Intelligent Systems, T\"ubingen, Germany\\
\textsuperscript{b}Information Sciences Institute, University of Southern California, USA\\
\textsuperscript{c}Centre for Quantum Computation and Intelligent Systems, University of Technology, Sydney, Australia}

\icmlkeywords{causality, causal inference, Granger, vector, autoregressive, VAR, process, hidden, latent, components, variables, channels}

\vskip 0.3in
]



\begin{abstract} 
A widely applied approach to causal inference from a non-experimental time series $X$, often referred to as ``(linear) Granger causal analysis'', is to regress present on past and interpret the regression matrix $\hat{B}$ causally.
However, if there is an unmeasured time series $Z$ that influences $X$, then this approach can lead to wrong causal conclusions, i.e., distinct from those one would draw if one had additional information such as $Z$.
In this paper we take a different approach: We assume that $X$ together with some hidden $Z$ forms a first order vector autoregressive (VAR) process with transition matrix $A$, and argue why it is more valid to interpret $A$ causally instead of $\hat{B}$.
Then we examine under which conditions the most important parts of $A$ are identifiable or almost identifiable from only $X$.
Essentially, sufficient conditions are (1) non-Gaussian, independent noise or (2) no influence from $X$ to $Z$.
We present two estimation algorithms that are tailored towards conditions (1) and (2), respectively, and evaluate them on synthetic and real-world data.
We discuss how to check the model using $X$.
\end{abstract}

\section{Introduction}
\label{sec::intro}

Inferring the causal structure of a stochastic dynamical system from a non-experimental time series of measurements is an important problem in many fields such as economics \citep{Luetkepohl2006} and neuroscience \citep{Goebel,Besserve2010}.

In the present paper, we approach this problem as follows: 
We assume that the measurements are a finite sample from a random process $X=(X_t)_{t \in \Z}$ which, together with another random process $Z=(Z_t)_{t \in \Z}$, forms a first order vector autoregressive (VAR) process. That is, $(X,Z)^\top$ obeys
\begin{align*}
 \left( \begin{array}{c} X_t \\ Z_t \end{array} \right)  = \left( \begin{array}{cc} B & C \\ D & E \end{array} \right) \left( \begin{array}{c} X_{t-1} \\ Z_{t-1} \end{array} \right) + N_t, \label{eqn::intro_VAR}
\end{align*}
for all $t \in \Z$, some matrices $B,C,D,E$, and some i.i.d.\ $N_i, i \in \Z$.  
So far this is a purely statistical model. 
Now we additionally assume that the variables in $Z$ correspond to real properties of the underlying system that are in principle measurable and intervenable.
Based on this we consider $B,C,D,E$ to have a \emph{causal meaning}. 
More precisely, we assume that $B$'s entries express the direct causal influences between the respective variables in $X$.
And more generally, we assume that for all variables in $(X,Z)^\top$ the matrices $B,C,D,E$ capture the respective direct and indirect causal influences.
Note that in this sense $C$ is particularly interesting because it tells which components of $X$ are jointly influenced by an unmeasured quantity, i.e., have a \defi{hidden confounder}, and how strong the influence is.

This way causal inference on $X$ is reduced to a \emph{statistical} problem: examining to what extent, i.e., under which assumptions, $B$ as well as $C,D,E$ are identifiable from the distribution of the process $X$, and how they can be estimated from a sample of $X$. 
It is worth mentioning that this approach can be justified in two different ways, following either \cite{Granger1969} or \cite{Pearl2000, Spirtes2000}.
We will briefly elaborate on this later (Section \ref{sec::link_causal_models}).

The first and main contribution of this paper is on the theoretical side: we present several results that show under which conditions $B$ and $C$ are identifiable or almost (i.e.\ up to a small number of possibilities) identifiable from only the distribution of $X$. 
Generally we assume that $Z$ has at most as many components as $X$. 
Theorem \ref{thm::idr} shows that if the noise terms are non-Gaussian and independent, and an additional genericity assumption holds true, then $B$ is uniquely identifiable.
Theorem \ref{thm::id_C} states that under the same assumption, those columns of $C$ that have at least two non-zero entries are identifiable up to scaling and permutation indeterminacies (because scale and ordering of the components of $Z$ are arbitrary). 
Theorem \ref{thm::almost_id} shows that regardless of the noise distribution (i.e., also in the case of Gaussian noise), if there is no influence from $X$ to $Z$ and an additional genericity assumption holds, then $B$ is identifiable from the covariance structure of $X$ up to a small finite number of possibilities. 
In Propositions \ref{prop::genericity1} and \ref{prop::genericity2} we prove that the additional assumptions we just called generic do in fact only exclude a Lebesgue null set from the parameter space.

The second contribution is a first examination of how the above identifiability results can be translated into estimation algorithms on finite samples of $X$.
We propose two algorithms. 
Algorithm \ref{alg::vem_algo}, which is tailored towards the conditions of Theorems \ref{thm::idr} and \ref{thm::id_C}, estimates $B$ and $C$ by approximately maximizing the likelihood of a parametric VAR model with a mixture of Gaussians as noise distribution.
Algorithm \ref{alg::cov_algo}, which is tailored towards the conditions of Theorem \ref{thm::almost_id}, estimates the matrix $B$ up to finitely many possibilities by solving a system of equations somewhat similar to the Yule-Walker equations \citep{Luetkepohl2006}. 
Furthermore, we briefly examine how the model assumptions that we make can to some extent be checked just based on the observed sample of $X$.
We examine the behavior of the two proposed algorithms on synthetic and real-world data.

It should be mentioned that probably the most widely applied approach to causal inference from time series data so far \citep{Luetkepohl2006}, which we refer to as \emph{practical Granger causal analysis} in this paper (often just called ``(linear) Granger causality''), is to simply perform a linear regression of present on past on the observed sample of $X$ and then interpret the regression matrix causally.
While this method may yield reasonable results in certain cases, it obviously can go wrong in others (see Section \ref{sec::pracitical_granger} for details).
We believe that the approach presented in this paper may in certain cases lead to more valid causal conclusions.

The remainder of this paper is organized as follows. 
In Section \ref{sec::related_work} we discuss related work.
In Section \ref{sec::notation} we introduce notation and definitions for time series.
In Section \ref{sec::prereq} we state the statistical and causal model that we assume throughout the paper. 
In Section \ref{sec::genres} we introduce the so-called generalized residual. 
Section \ref{sec::all_id_results} contains the three main results on identifiability (Theorems \ref{thm::idr} to \ref{thm::almost_id}) as well as arguments for the genericity of certain assumptions we need to make (Propositions \ref{prop::genericity1} and \ref{prop::genericity2}).
In Section \ref{sec::algorithms} we present the two estimation algorithms and discuss model checking.
Section \ref{sec::exp} contains experiments for Algorithms \ref{alg::vem_algo} and \ref{alg::cov_algo}. 
We conclude with Section \ref{sec::conclusion}.

\section{Related Work}
\label{sec::related_work}

We briefly discuss how the present work is related to previous papers in similar directions.



\stress{Inference of properties of processes with hidden components:}
The work \citep{Jalali2012} also assumes a VAR model with hidden components and tries to identify parts of the transition matrix.
However their results are based on different assumptions: they assume a ``local-global structure'', i.e., connections between observed components are sparse and each latent series interacts with many observed components, to achieve identifiability.
The authors of \citep{Boyen1999} - similar to us - apply a method based on expectation maximization (EM) to infer properties of partially observed Markov processes.
Unlike us, they consider finite-state Markov processes and do not provide a theoretical analysis of conditions for identifiability. 
The paper \citep{Etesami2012} examines identifiability of partially observed processes that have a certain tree-structure, using so-called discrepancy measures.

\stress{Harnessing non-Gaussian noise for causal inference:}
The paper \citep{Hyv} uses non-Gaussian noise to infer instantaneous effects.
In \citep{Hoyer2008}, the authors use the theory underlying overcomplete independent component analysis (ICA) \citep[Theorem 10.3.1]{Kagan73} to derive identifiability (up to finitely many possibilities) of linear models with hidden variables, which is somewhat similar to our Theorem \ref{thm::idr}. 
However, there are two major differences: First, they only consider models which consist of finitely many observables which are mixtures of finitely many noise variables. 
Therefore their results are not directly applicable to VAR models.
Second, they show identifiability only up to a finite number of possibilities, while we (exploiting the autoregressive structure) prove unique identifiability.

\stress{Integrating several definitions of causation:}
The work \citep{Eichler2012} provides an overview over various definitions of causation w.r.t.\ time series, somewhat similar to but more comprehensive than our brief discussion in Sections \ref{sec::link_causal_models} and \ref{sec::pracitical_granger}.

\section{Time Series: Notation and Definitions}
\label{sec::notation}

Here we introduce notation and definitions w.r.t.\ time series.
We denote multivariate \defi{time series}, i.e., families of random vectors over the index set $\Z$, by upper case letters such as $X$.
As usual, $X_t$ denotes the $t$-th member of $X$, and $X_t^k$ denotes the $k$-th component of the random vector $X_t$.
Slightly overloading terminology, we call the univariate time series $X^k=(X^k_t)_{t \in \Z}$ the \defi{$k$-th component of $X$}.
By $P_X$ we denote the distribution of the random process $X$, i.e., the joint distribution of all $X_t$, $t \in \Z$.

Given a $K_X$-variate time series $X$ and a $K_Z$-variate time series $Z$, $( X, Z)^\top$ denotes the $(K_X + K_Z)$-variate series 
\[ \left( (X_t^1, \ldots, X_t^{K_X}, Z_t^1, \ldots, Z_t^{K_Z})^\top \right)_{t \in \Z} . \]

A $K$-variate time series $W$ is a \defi{vector autoregressive process (of order $1$), or VAR process for short, with VAR transition matrix $A$ and noise covariance matrix $\Sigma$}, if it allows a \defi{VAR representation}, i.e.,
\begin{align}
 W_t = A W_{t-1} + N_t, \label{eqn:complete_VAR}
\end{align}
the absolute value of all eigenvalues of $A$ is less than\footnote{We require all VAR processes to be stable \citep{Luetkepohl2006}.} 1,
and $N$ is an i.i.d. noise time series such that $\cov(N_0) = \Sigma$. 
We say $W$ is a \defi{diagonal-structural VAR process} if in the above definition the additional condition is met that $N_0^1, \ldots, N_0^K$ are jointly independent.\footnote{Note that the notion ``diagonal-structural'' is a special case of the more general notion of ``structural'' in e.g., \citep{Luetkepohl2006}.}

\section{Statistical and Causal Model Assumptions}
\label{sec::prereq}

In this section we introduce the statistical model that we consider throughout the paper and discuss based on which assumptions its parameters can be interpreted causally.
Moreover, we give an example for how practical Granger causal analysis can go wrong.

\subsection{Statistical Model}
\label{sec::model}

Let $K_X$ be arbitrary but fixed. Let $X$ be a $K_X$-variate time series. 
As stated in Section \ref{sec::intro}, $X$ is the random process from which we assume we measured a sample. 
In particular, the random variables in $X$ have a meaning in reality (e.g., $X^1_3$ is the temperature measured in room 1 at time 3) and we are interested in the causal relations between these variables.
Let $X$ be related to a $K$-variate VAR process $W$, with transition matrix $A$, noise time series $N$, and noise covariance matrix $\Sigma$, and a $K_Z$-variate time series $Z$, as follows:
$W = (X,Z)^\top$ and $K_Z \leq K_X$.
Furthermore, let
\begin{align}
 A =: \left( \begin{array}{cc} B & C \\ D & E \end{array} \right), \label{eqn::var}
\end{align}
with $B$ a $K_X \times K_X$ matrix. 
We call $B$, the most interesting part of $A$, the \emph{structural matrix underlying $X$}. 
Furthermore, in case $C \neq 0$, we call $Z$ a \defi{hidden confounder}.

\subsection{Causal Assumptions} 
\label{sec::link_causal_models}

As already mentioned in Section \ref{sec::intro}, throughout this paper we assume that there is an underlying system such that all variables in $W$ correspond to actual properties of that system which are in principle measurable and intervenable.
While we assume that a finite part of $X$ was in fact measured (Section \ref{sec::model}), $Z$ is completely unmeasured.
Furthermore we assume that the entries of $A$, in particular the submatrix $B$, capture the actual non-instantaneous causal influences between the variables in $W$.
We also mentioned that there are two lines of thought that justify this assumption. 
We briefly elaborate on this here.

On the one hand, \citep{Granger1969} proposed a definition of causation between observables which we will refer to as \defi{Granger's ideal definition}.
Assume the statistical model for the observed sample of $X$ specified in Section \ref{sec::model}.
If we additionally assume that $Z$ correctly models the whole rest of the universe or the ``relevant'' subpart of it, then according to Granger's ideal definition the non-instantaneous (direct) causal influences between the components of $X$ are precisely given by the entries of $B$.
But this implies that everything about $B$ that we can infer from $X$ can be interpreted causally, if one accepts Granger's ideal definition and the additional assumptions that are necessary (such as $K_Z \leq K_X$, which in fact may be a quite strong assumption of course).
This is one way to justify our approach.

On the other hand, \citep{Pearl2000} does not define causation based on measurables alone but instead formalizes causation by so-called structural equation models (SEMs) and links them to observable distributions via additional assumptions.
In this sense, let us assume that $W$ forms a causally sufficient set of variables, whose correct structural equations are given by the VAR equations (\ref{eqn:complete_VAR}), i.e., these equations represent actual causal influences from the r.h.s.\ to the l.h.s.\footnote{Note that here we ignore the fact that Pearl generally only considers models with finitely many variables while the process $W$ is a family of infinitely many (real-valued) variables.}
In particular these equations induce the correct (temporal) causal directed acyclic graph (DAG) for $(X,Z)^\top$.
Then, essentially following the above mentioned author, everything about $B$ that we can infer from the distribution of $X$ can be interpreted causally.
This is the other way to justify our approach (in case the requirement $K_Z \leq K_X$ and the other assumptions are met).
It is important to mention that the usual interpretation of SEMs is that they model the mechanisms which generate the data and that they predict the outcomes of randomized experiments w.r.t.\ the variables contained in the equations.


\subsection{Relation to Practical Granger Causal Analysis and How It Can Go Wrong}
\label{sec::pracitical_granger}

The above ideal definition of causation by Granger (Section \ref{sec::link_causal_models}) needs to be contrasted with what we introduced as ``practical Granger causal analysis'' in Section \ref{sec::intro}.
In practical Granger causal analysis, one just performs a linear regression of present on past on the observed $X$ and then interprets the regression matrix causally.\footnote{We are aware that nonlinear models~\cite{Chu08} and nonparametric estimators~\cite{Schreiber00} have been used to find temporal causal relations. In this paper we focus on the linear case.}
While making the ideal definition practically feasible, this may lead to wrong causal conclusions in the sense that it does not comply with the causal structure that we would infer given we had more information.\footnote{Obviously, if one is willing to assume that $X$ is causally sufficient already, then the practical Granger causation can be justified along the lines of Section \ref{sec::link_causal_models}.}

Let us give an example for this. 
Let $X$ be bivariate and $Z$ be univariate. 
Moreover, assume
$$ \newcommand*{\temp}{\multicolumn{1}{r|}{}}
 A = \left( \begin{array}{cccc} 0.9 & 0 &\temp & 0.5 \\ 0.1 & 0.1 &\temp & 0.8 \\ \hline 0 & 0 &\temp & 0.9 \\ \end{array} \right), $$
and let the covariance matrix of $N_t$ be the identity matrix. 
To perform practical Granger causal analysis, we proceed as usual: we fit a VAR model on \emph{only} $X$, in particular compute, w.l.o.g.\ assuming zero mean, the transition matrix by 
\begin{align}
 B_{\text{pG}} := \E(X_t X_{t-1}^\top) \E(X_t X_t^\top)^{-1} = \left( \begin{array}{cc} 0.89 & 0.35\\ 0.08 & 0.65\\ \end{array} \right) \label{eqn::BEE}
\end{align} 
(up to rounding) and interpret the coefficients of $B_{\text{pG}}$ as causal influences.
Although, based on $A$, $X^2_{t}$ does in fact not cause $X^1_{t+1}$, $B_{\text{pG}}$ suggests that there is a strong causal effect $X^2_{t} \rightarrow X^1_{t+1}$ with the strength $0.35$. 
It is even stronger than the relation $X^1_{t} \rightarrow X^2_{t+1}$, which actually exists in the complete model with the strength $0.1$.



\section{The Generalized Residual: Definition and Properties}
\label{sec::genres}

In this section we define the generalized residual and discuss some of its properties.
The generalized residual is used in the proofs of the three main results of this paper, Theorems \ref{thm::idr} to \ref{thm::almost_id}.


For any $K_X \times K_X$ matrices $U_1, U_2$ let
\begin{align*}
 R_t(U_1,U_2) &:= X_t - U_1 X_{t-1} - U_2 X_{t-2} .
\end{align*}
We call this family of random vectors \emph{generalized residual}.
Furthermore let
\begin{align*}
 M_1 
 &:= \E \left[ W_t \cdot ( X_t^\top , X_{t-1}^\top ) \right] .
\end{align*}


In what follows, we list some simple properties of the generalized residual. 
Proofs can be found in Section~\ref{sec::genres_proofs}.

\begin{Lemma}
\label{lem::repr_R}
We have
\begin{align}
 R_t(U_1,U_2) &= (B^2 + C D - U_1 B - U_2 ) X_{t-2} \notag \\
 &\phantom{=} + (B C + C E - U_1 C) Z_{t-2} \notag \\
 &\phantom{=} + (B - U_1) N^X_{t-1} + C N^Z_{t-1} + N^X_{t} , \label{eqn::gres}
\end{align}
if $K > K_X$. In case $K = K_X$, the same equation holds except that one sets $C:=D:=E:=0$. 

\end{Lemma}


\begin{Lemma}
\label{lem::zero_impl}
If $(U_1,U_2)$ satisfies the equation
\begin{align}
 \left( U_1 , U_2 \right) \left( \begin{array}{cc}  B &  C  \\ \I & 0 \end{array} \right) = \left( B^2 + C D , B C + C E \right) , \label{eqn::coeff_zero}
\end{align}
then $R_t(U_1,U_2)$ is independent of $( X_{t-2-j} )_{j=0}^\infty$, and in particular, for $j \geq 0$,
\begin{align}
\cov(R_t(U_1,U_2) , X_{t-2-j}) = 0 . \label{eqn::cov_zero}
\end{align}
\end{Lemma}

Let $\Gamma^X_i := \cov(X_t, X_{t-i})$ for all $i$. That is, $\Gamma^X_i$ are the \defi{autocovariance matrices} of $X$.
Note that equation (\ref{eqn::cov_zero}), for $j=0,1$, can equivalently be written as the single equation
\begin{align}
 \left( U_1 , U_2 \right) \left( \begin{array}{cc}  \Gamma^X_{1} &  \Gamma^X_{2}  \\ \Gamma^X_{0} & \Gamma^X_{1} \end{array} \right) = \left( \Gamma^X_{2} , \Gamma^X_{3} \right) . \label{eqn::coeff_zero_ver}
\end{align}


Keep in mind that, as usual, we say a $m \times n$ matrix has \defi{full rank} if its (row and column) rank equals $\min \{m,n\}$.

\begin{Lemma}
\label{lem::uncorrimpl}
Let $M_1$ have full rank.
If $(U_1,U_2)$ satisfies equation (\ref{eqn::cov_zero}) for $j=0,1$, then it satisfies equation (\ref{eqn::coeff_zero}).
\end{Lemma}

\begin{Lemma}
\label{lem::U_exist}
If $K=K_X$ or if $C$ has full rank, then there exists $(U_1,U_2)$ that satisfies equation (\ref{eqn::coeff_zero}).
\end{Lemma}

\section{Theorems on Identifiability and Almost Identifiability}
\label{sec::all_id_results}

This section contains the main results of the present paper.
We present three theorems on identifiability and almost identifiability of $B$ and $C$ (defined in Section \ref{sec::model}), respectively, given $X$ and briefly argue why certain assumptions we have to make can be considered as generic.
Recall the definition of the matrix $M_1$ in Section \ref{sec::genres}.
Note that the following results show (almost) identifiability of $B$ for all numbers $K_Z$ of hidden components \emph{simultaneously}, as long as $0 \leq K_Z \leq K_X$ (which contains the case of no hidden components as a special case).

\subsection{Assuming Non-Gaussian, Independent Noise}
\label{sec::idr}


We will need the following assumptions for the theorems. 
\begin{Assumptions}
We define the following abbreviations for the respective subsequent assumptions.
\begin{itemize}
\item[\emph{A1}:] All noise terms $N^k_t$, $k=1,\ldots,K, t \in \Z$, are non-Gaussian. 
\item[\emph{A2}:] $W$ is a diagonal-structural VAR process (as defined in Section \ref{sec::notation}).
\item[\emph{G1}:] $C$ (if it is defined, i.e., if $K > K_X$) and $M_1$ have full rank.
\end{itemize}
\end{Assumptions}
(We will discuss the genericity of G1 in Section \ref{sec::genericity}.)

The following definition of $F_1$ is not necessary for an intuitive understanding, but is needed for a precise formulation of the subsequent identifiability statements. 
Let $F_1$ denote the set of all $K'$-variate VAR processes $W'$ with $K_X \leq K' \leq 2 K_X$ (i.e.\ $W$ has at most as many hidden components as observed ones), which satisfy the following properties w.r.t. $N', C', M_1'$ (defined similarly to $N, C, M_1$ in Section \ref{sec::prereq}): 
assumptions A1, A2 and G1 applied to $N', C', M_1'$ (instead of $N, C, M_1$) hold true.


\begin{Theorem}
\label{thm::idr}
If assumptions A1, A2 and G1 hold true, then $B$ is uniquely identifiable from only $P_X$.

That is: There is a map $f$ such that for each $W' \in F_1$, and $X'$ defined as the first $K_X$ components of $W'$, $f(P_{X'})=B'$ iff $B'$ is the structural matrix underlying $X'$.
\end{Theorem}

A detailed proof can be found in Section~\ref{sec::thm1}.
%
%
The idea is to chose $U_1,U_2$ such that $R_t(U_1,U_2)$ is a linear mixture of only \emph{finitely} many noise terms, which is possible based on Lemmas \ref{lem::repr_R} to \ref{lem::U_exist}.
Then, using the identifiability result underlying overcomplete ICA \citep[Theorem 10.3.1]{Kagan73}, the structure of the mixing matrix of $(R_t(U_1,U_2), R_{t-1}(U_1,U_2))^\top$ allows to uniquely determine $B$ from it.

Again using \citep[Theorem 10.3.1]{Kagan73}, one can also show the following result.
For a matrix $M$ let $S(M)$ denote the set of those columns of $M$ that have at least two non-zero entries, and if $M$ is not defined, let $S(M)$ denote the empty set.
A proof can be found in Section~\ref{sec::thm2}.
\begin{Theorem}
\label{thm::id_C}
If assumptions A1, A2 and G1 hold true, then the set of columns of $C$ with at least two non-zero entries is identifiable from only $P_X$ up to scaling of those columns.

In other words: There is a map $f$ such that for each $W' \in F_1$ with $K'$ components, $X'$ defined as the first $K_X$ components of $W'$, and $C'$ defined as the upper right $K_X \times (K' - K_X)$ submatrix of the transition matrix of $W'$, $f(P_{X'})$ coincides with $S(C')$ up to scaling of its elements.

\end{Theorem}

\subsection{Assuming $D=0$}
\label{sec::id_D_zero}

In this section we present a theorem on the almost identifiability of $B$ under different assumptions. In particular, we drop the non-Gaussianity assumption. 
Instead, we make the assumption that $Z$ is not influenced by $X$, i.e., $D=0$.


Given $U=(U_1, U_2)$, let 
\begin{align}
T_U(Q) := Q^2 - U_1 Q - U_2 ,
\end{align}
for all square matrices $Q$ that have the same dimension as $U_1$.
Slightly overloading notation, we let $T_U(\alpha) := T_U(\alpha \I)$ for all scalars $\alpha$.
Note that $\det(T_U(\alpha))$ is a univariate polynomial in $\alpha$.


We will need the following assumptions for the theorem.  
\begin{Assumptions}
We define the following abbreviations for the respective subsequent assumptions.
\begin{itemize}
\item[\defi{A3}:] $D = 0$.
\item[\defi{G2}:] The transition matrix $A$ is such that there exists $U=(U_1, U_2)$ such that equation (\ref{eqn::coeff_zero}) is satisfied and $\det(T_U(\alpha))$ has $2 K_X$ distinct roots.
\end{itemize}
\end{Assumptions}
(We will discuss the genericity of G2 in Section \ref{sec::genericity}.)

The following definition of $F_2$ is not necessary for an intuitive understanding, but is needed for a precise formulation of the subsequent identifiability statement. 
Let $F_2$ denote the set of all $K'$-variate VAR processes $W'$ with $K_X \leq K' \leq 2 K_X$, which satisfy the following properties w.r.t. $N', A', C', D', M_1'$ (defined similarly to $N, A, C, D, M_1$ in Section \ref{sec::prereq}): 
assumptions A3, G1 and G2 applied to $N', A', C', D', M_1'$ (instead of $N, A, C, D, M_1$) hold true.

\begin{Theorem}
\label{thm::almost_id}
If assumptions A3, G1 and G2 hold true, then $B$ is identifiable from only the covariance structure of $X$ up to $\binom{2 K_X}{K_X}$ possibilities.

In other words: There is a map $f$ such that for each $W' \in F_2$, and $X'$ defined as the first $K_X$ components of $W'$, $f(X')$ is a set of at most $\binom{2 K_X}{K_X}$ many matrices, and $B' \in f(P_{X'})$ for $B'$ the structural matrix underlying $X'$.
\end{Theorem}

A detailed proof can be found in Section~\ref{sec::thm3}.
The proof idea is the following: Let $L$ denote the set of all $(U, \tilde{B})$, with $U = (U_1, U_2)$, that satisfy equation (\ref{eqn::cov_zero}) for $j=0,1$, as well as the equation
\begin{align}
\label{eqn::T_zero}
 T_{U}(\tilde{B}) &= 0, 
\end{align}
and meet the condition that $\det(T_U(\alpha))$ has $2 K_X$ distinct roots.
$L$ is non-empty and $(U,B)$ is an element of it, for the true $B$ and some $U$, due to Lemmas \ref{lem::zero_impl} to \ref{lem::U_exist}.
But $L$ is only defined based on the covariance of $X$ and has at most $\binom{2 K_X}{K_X}$ elements (based on \citep{Dennis1976}).

Note the similarity between equation (\ref{eqn::cov_zero}), or its equivalent, equation (\ref{eqn::coeff_zero_ver}), and the well-known Yule-Walker equation \citep{Luetkepohl2006}. The Yule-Walker equation (which is implicitly used in equation (\ref{eqn::BEE})) determines $B$ uniquely under some genericity assumption and given $C=0$.

\subsection{Discussion on the Genericity of Assumptions G1 and G2}
\label{sec::genericity}

In this section we want to briefly argue why the assumptions G1 and G2 are generic.
A detailed elaboration with precise definitions and proofs can be found in Section~\ref{sec::genericity_proofs}. 
The idea is to define a natural parametrization of $(A,\Sigma)$ and to show that the restrictions that assumptions G1 and G2, respectively, impose on $(A,\Sigma)$ just exclude a Lebesgue null set in the natural parameter space and thus can be considered as generic.

In this section, let $K$ such that $K_X \leq K \leq 2K_X$ be arbitrary but fixed.
Let $\lambda_k$ denote the $k$-dimensional Lebesgue measure on $\R^k$.

Let $\Theta_1$ denote the set of all possible parameters $(A',\Sigma')$ for a $K$-variate VAR processes $W'$ that additionally satisfy assumption A2, i.e., correspond to structural $W'$.
Let $S_1$ denote the subset of those $(A',\Sigma') \in \Theta_1$ for which also assumption G1 is satisfied.
And let $g$ denote the natural parametrization of $\Theta_1$ which is defined in Section~\ref{sec::genericity1}.

\begin{Proposition}
\label{prop::genericity1}
We have $\lambda_{K^2 + K} \left( g^{-1}(\Theta_1 \setminus S_1) \right) = 0$.
\end{Proposition}

A proof can be found in Section~\ref{sec::genericity1}. 
The proof idea is that $g^{-1}(\Theta_1 \setminus S_1)$ is essentially contained in the union of the root sets of finitely many multivariate polynomials and hence is a Lebesgue null set.


Let $\Theta_2$ denote the set of all possible parameters $(A',\Sigma')$ for the $K$-variate VAR processes $W$ that additionally satisfy assumption A3, i.e., are such that the submatrix $D$ of $A$ is zero.
Let $S_2$ denote the subset of those $(A',\Sigma') \in \Theta_2$ for which also assumptions G1 and G2 are satisfied.
Let $h$ denote the natural parametrization of $\Theta_1$ which is defined in Section~\ref{sec::genericity2}.
A proof for the following proposition (which is based on a similar idea as that of Proposition~\ref{prop::genericity1}) can also be found in Section~\ref{sec::genericity2}. 

\begin{Proposition}
\label{prop::genericity2}
We have $\lambda_{2 K^2 - K_X K_Z} \left( h^{-1}(\Theta_2 \setminus S_2) \right) = 0$.
\end{Proposition}

\section{Estimation Algorithms}
\label{sec::algorithms}

In this section we examine how the identifiability results in Section \ref{sec::all_id_results} can be translated into estimators on finite data.
We propose two algorithms.

\subsection{Algorithm Based on Variational EM}
\label{sec::vem_algo}

\begin{algorithm}[tb]
   \caption{Estimate $B,C$ using variational EM}
   \label{alg::vem_algo}
\begin{algorithmic}[1]
   \STATE {\bfseries Input:} Sample $x_{1:L}$ of $X_{1:L}$.
   \STATE {Initialize the transition matrix and the parameters of the Gaussian mixture model, denoted as $\theta^0$, set $j\leftarrow 0$.}
   \REPEAT
   \STATE {{\bfseries E step:} Evaluate
   \begin{align*}q^j(z_{1:L},v^X_{1:L},v^Z_{1:L})=q^j(z_{1:L})q^j(v^X_{1:L}) q^j(v^Z_{1:L}) ,\end{align*} 
   which is the variational approx. to the true posterior $q^j(z_{1:L},v^X_{1:L},v^Z_{1:L}|x_{1:L})$, by maximizing the variational lower bound, i.e., $q^j = \text{arg}\max_{q}\mathcal{L}(q,\theta^j)$}.
   \STATE {{\bfseries M step:} Evaluate $\theta^{j+1}=\text{arg}\max_{\theta}\mathcal{L}(q^j,\theta)$.}
   \STATE {$j\leftarrow j+1$.}
   \UNTIL{convergence}
   \STATE {{\bfseries Output:} The final $\theta^j$, containing the estimated $B,C$.} 
\end{algorithmic}
\end{algorithm}

Here we present an algorithm for estimating $B$ and $C$ which is closely related to Theorems \ref{thm::idr} and \ref{thm::id_C}.
Keep in mind that the latter theorem in fact only states identifiability for $S(C)$ (defined in Section \ref{sec::id_D_zero}), up to scaling, not for the exact $C$. 
The idea is the following: 
We transform the model of $X$ underlying these theorems 
(i.e.\ the general model from Section \ref{sec::model} together with assumptions A1, A2 and G1 from Section \ref{sec::idr}) 
into a parametric model by assuming the noise terms $N^k_t$ to be mixtures of Gaussians.\footnote{Obviously, Theorems \ref{thm::idr} and \ref{thm::id_C} also imply identifiability of $B$ and (up to scaling) $S(C)$ for this parametric model.
We conjecture that this implies consistency of the (non-approximate) maximum likelihood estimator for that model under appropriate assumptions.}
Then we estimate all parameters, including $B$ and $C$, by approximately maximizing the likelihood of the given sample of $X$ using a variational expectation maximization (EM) approach similar to the one in \citep{Oh05avariational}.
(Directly maximizing the likelihood is intractable due to the hidden variables ($Z$ and mixture components) that have to be marginalized out.)
Let $y_{1:L}$ be shorthand for $(y_1,\ldots,y_L)$.
The estimator is outlined by Algorithm \ref{alg::vem_algo}, where $(V^X_t, V^Z_t)$ with values $(v^X_t,v^Z_t)$ denote the vectors of mixture components for $N^X_t$ and $N^Z_t$, respectively; $q^j(z_{1:L},v^X_{1:L},v^Z_{1:L} | x_{1:L})$ the true posterior of $Z_{1:L},V^X_{1:L},V^Z_{1:L}$ under the respective parameter vector $\theta^j$ (which comprises $A, \Sigma$ as well as the Gaussian mixture parameters) at step $j$; and $\mathcal{L}$ the variational lower bound.
The detailed algorithm can be found in Section~\ref{sec::vem_algo_details}.
Note that, if needed, one may use cross validation as a heuristic to determine $K_Z$ and the number of Gaussian mixture components.



\subsection{Algorithm Based on the Covariance Structure}
\label{sec::cov_algo}

\begin{algorithm}[tb]
   \caption{Estimate $B$ using covariance structure}
   \label{alg::cov_algo}
\begin{algorithmic}[1]
   \STATE {\bfseries Input:} Sample $x_{1:L}$ of $X_{1:L}$.
   \STATE Solve the equation (\ref{eqn::coeff_zero_ver}), with $\Gamma^X_i$ replaced by $\hat{\Gamma}^X_i$. Let $(\hat{U}_1,\hat{U}_2)$ denote the solution.
   \STATE Solve equation (\ref{eqn::T_zero}) with $U := (\hat{U}_1,\hat{U}_2)$ for $\tilde{B}$. Let $\hat{B}_1, \ldots, \hat{B}_n$ denote the solvents. 
   \STATE {\bfseries Output:} $\hat{B}_1, \ldots, \hat{B}_n$.
\end{algorithmic}
\end{algorithm}

Now we present an algorithm, closely related to Theorem \ref{thm::almost_id}, for estimating $B$ up to finitely many possibilities. 
It relies on the proof idea of that theorem, as we outlined it at the end of Section \ref{sec::id_D_zero}, and it is meant to be applied for cases where the conditions of that theorem are met.
It uses only the estimated autocovariance structure of $X$. 
Keep in mind that $\hat{\Gamma}^X_i$ denote the sample autocovariance matrices (similar to the true autocovariances $\Gamma^X_i$ defined in Section \ref{sec::genres}). 
The estimation algorithm is given by Algorithm \ref{alg::cov_algo}.

\subsection{Model Checking}
\label{sec::model_checking}


Ideally we would like to know whether the various model assumptions we make in this paper, most importantly the one that the entries of $B$ can in fact be interpreted causally, are appropriate. 
Obviously, this is impossible to answer just based on the observed sample of $X$.
Nonetheless one can check these assumptions to the extent they imply testable properties of $X$. 

For instance, to check (to a limited extent) the assumptions underlying Theorems \ref{thm::idr} and \ref{thm::id_C} and Algorithm \ref{alg::vem_algo}, i.e., the general statistical and causal model assumptions from Sections \ref{sec::model} and \ref{sec::link_causal_models} together with A1, A2 and G1 from Section \ref{sec::idr}, we propose the following two tests:
First, test whether $R_t(\hat{U}_1,\hat{U}_2)$ is independent of $( X_{t-2-j} )_{j=0}^J$, for $(\hat{U}_1,\hat{U}_2)$ as defined in Algorithm \ref{alg::cov_algo}, and for say $J=2$. (If Algorithm \ref{alg::cov_algo} finds no $(\hat{U}_1,\hat{U}_2)$ then the test is already failed.)
Second, check whether all components of $X_t$ are non-Gaussian using e.g.\ the Kolmogorov-Smirnov test \citep{conover1971} for Gaussianity.

Note that under the mentioned assumptions, both properties of $X$ do in fact hold true.
Regarding the independence statement, this follows from Lemmas \ref{lem::U_exist} and \ref{lem::zero_impl}.
W.r.t.\ the non-Gaussianity statement, this follows from the fact \citep[Theorem 7.8]{Ramachandran1967} that the distribution of an infinite weighted sum of non-Gaussian random variables is again non-Gaussian.
It should be mentioned that the first test can also be used to check (to a limited extent) the assumptions underlying Theorem \ref{thm::almost_id} and Algorithm \ref{alg::cov_algo}.

\section{Experiments}
\label{sec::exp}

In this section we evaluate the two algorithms proposed in Section \ref{sec::algorithms} on synthetic and real-world data and compare them to the practical Granger causation estimator. 
Keep in mind that the latter is defined by replacing the covariances in equation (\ref{eqn::BEE}) by sample covariances.

\subsection{Synthetic Data}

We empirically study the behavior of Algorithms \ref{alg::vem_algo} and \ref{alg::cov_algo} on simulated data, in dependence on the sample length.
Note that, based on theoretical considerations (see Section \ref{sec::pracitical_granger}), it can be expected that the error of the practical Granger estimator is substantially bounded away from zero in the generic case.

\subsubsection{Algorithm \ref{alg::vem_algo}}

Here we evaluate Algorithm \ref{alg::vem_algo}.

{\bf Experimental setup:} We consider the case of a 2-variate $X$ and a 1-variate $Z$, i.e., $K_X = 2, K_Z = 1$.
We consider sample lengths $L= 100, 500, 1000, 5000$ and for each sample length we do 20 runs.
In each run we draw the matrix $A$ uniformly at random from the stable matrices and then randomly draw a sample of length $L$ from a VAR process $W=(X,Z)^\top$ with $A$ as transition matrix and noise $N_t^k$ distributed according to a super-Gaussian mixtures of Gaussians.
Then we apply Algorithm \ref{alg::vem_algo} and the practical Granger causation estimator on the sample of \emph{only $X$}.

{\bf Outcome:} We calculated the root-mean-square error (RMSE) of Algorithm \ref{alg::vem_algo}, i.e., 
$ \frac{1}{20} \sum_{n=1}^{20} ( B^{\text{est}}_n - B^{\text{true}}_n )^2$,
where $B^{\text{est}}_n, B^{\text{true}}_n$ denotes the output of Algorithm \ref{alg::vem_algo} and the true $B$, respectively, for each run $n$.
The RMSE as a function of the sample length $L$ is depicted in Figure \ref{fig::exp2}, along with the RMSE of the practical Granger algorithm.
 
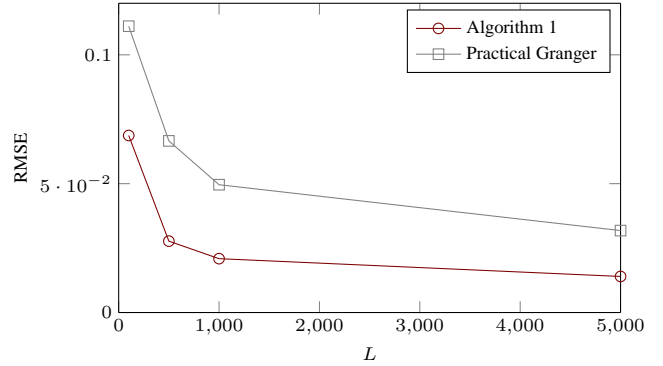
\begin{figure}
  \centering
  \setlength{\figwidth}{.85\columnwidth}
  \setlength{\figheight}{.18\textheight}
  {\scriptsize %
   \beginpgfgraphicnamed{figure2_20_runs-external}%
%
%
\definecolor{mycolor1}{rgb}{0.49060,0.00000,0.00000}%
\begin{tikzpicture}

\begin{axis}[%
width=0.95092\figwidth,
height=\figheight,
at={(0\figwidth,0\figheight)},
scale only axis,
separate axis lines,
every outer x axis line/.append style={black},
every x tick label/.append style={font=\color{black}},
xmin=0,
xmax=5000,
xlabel={$L$},
every outer y axis line/.append style={black},
every y tick label/.append style={font=\color{black}},
ymin=0,
ymax=0.12,
ylabel={RMSE},
legend style={legend cell align=left,align=left,draw=black}
]
\addplot [color=mycolor1,solid,mark=o,mark options={solid}]
  table[row sep=crcr]{%
100	0.0687\\
500	0.0277\\
1000	0.0209\\
5000	0.014\\
};
\addlegendentry{Algorithm 1};

\addplot [color=gray,solid,mark=square,mark options={solid}]
  table[row sep=crcr]{%
100	0.1111\\
500	0.0666\\
1000	0.0496\\
5000	0.0318\\
};
\addlegendentry{Practical Granger};

\end{axis}
\end{tikzpicture}%
%
   \endpgfgraphicnamed%
 }
  \caption{RMSE of Algorithm \ref{alg::vem_algo} and the practical Granger estimator as a function of sample length $L$.\label{fig::exp2}} 
\end{figure}

{\bf Discussion:} This suggests that for $L \to \infty$ the error of Algorithm \ref{alg::vem_algo} is negligible, although it may not converge to zero. 
The error of the practical Granger estimator for $L \to \infty$ is still small but substantially bigger than that of Algorithm \ref{alg::vem_algo}.

\subsubsection{Algorithm \ref{alg::cov_algo}}
\label{sec::synth_cov_algo}


Here we empirically establish the error of Algorithm \ref{alg::cov_algo}, more precisely the deviation between the true $B$ and the best out of the several estimates that Algorithm \ref{alg::cov_algo} outputs. 
Obviously in general it is unknown which of the outputs of Algorithm \ref{alg::cov_algo} is the best estimate. However here we rather want to establish that asymptotically, the output of Algorithm \ref{alg::cov_algo} in fact contains the true $B$.
Also we compare Algorithm \ref{alg::cov_algo} to the practical Granger estimator, although it needs to be said, that the latter is usually not applied to univariate time series.

{\bf Experimental setup:} We consider the case of 1-variate $X$ and $Z$, i.e., $K_X = K_Z = 1$.
We consider sample lengths $L= 10^1, 10^2, \ldots, 10^7$ and for each sample length we do 20 runs.
In each run we draw the matrix $A$ uniformly at random from the stable matrices with the constraint that the lower left entry is zero and then randomly draw a sample of length $L$ from a VAR process $W=(X,Z)^\top$ with $A$ as transition matrix and standard normally distributed noise $N$.
Then we apply Algorithm \ref{alg::cov_algo} and the practical Granger causation estimator on the sample of only $X$.

{\bf Outcome:} We calculated the root-mean-square error (RMSE) of Algorithm \ref{alg::cov_algo}, i.e., 
$\frac{1}{20} \sum_{n=1}^{20} ( B^{\text{best est}}_n - B^{\text{true}}_n )^2$,
where $B^{\text{best est}}_n, B^{\text{true}}_n$ denotes the best estimate for $B$ returned by Algorithm \ref{alg::cov_algo} (i.e., the one out of the two outputs that minimizes the RMSE) and true $B$ for each run $n$, respectively. The RMSE as a function of the sample length $L$ is depicted in Figure~\ref{fig::exp1}, along with the RMSE of the practical Granger estimator.

\begin{figure}
  \centering
  \setlength{\figwidth}{.85\columnwidth}
  \setlength{\figheight}{.18\textheight}
  {\scriptsize %
   \beginpgfgraphicnamed{figure1-external}%
%
%
\definecolor{mycolor1}{rgb}{0.49060,0.00000,0.00000}%
\begin{tikzpicture}

\begin{axis}[%
width=0.95092\figwidth,
height=\figheight,
at={(0\figwidth,0\figheight)},
scale only axis,
separate axis lines,
every outer x axis line/.append style={black},
every x tick label/.append style={font=\color{black}},
xmode=log,
xmin=10,
xmax=10000000,
xminorticks=true,
xlabel={$L$},
every outer y axis line/.append style={black},
every y tick label/.append style={font=\color{black}},
ymin=0,
ymax=0.5,
ylabel={RMSE},
legend style={legend cell align=left,align=left,draw=black}
]
\addplot [color=mycolor1,solid,mark=o,mark options={solid}]
  table[row sep=crcr]{%
10	0.486000461104325\\
100	0.199910513660451\\
1000	0.077014759328329\\
10000	0.0564719472428577\\
100000	0.0620831156830964\\
1000000	0.0150246754235297\\
10000000	0.00992261981206805\\
};
\addlegendentry{Algorithm 2};

\addplot [color=gray,solid,mark=square,mark options={solid}]
  table[row sep=crcr]{%
10	0.420624246437358\\
100	0.21184208648182\\
1000	0.136000618507603\\
10000	0.237569786717835\\
100000	0.118459902431219\\
1000000	0.144878270554243\\
10000000	0.188257268958586\\
};
\addlegendentry{Practical Granger};

\end{axis}
\end{tikzpicture}%
%
   \endpgfgraphicnamed%
 }
  \caption{RMSE of Algorithm \ref{alg::cov_algo} and the practical Granger estimator as a function of sample length $L$.\label{fig::exp1}} 
\end{figure}
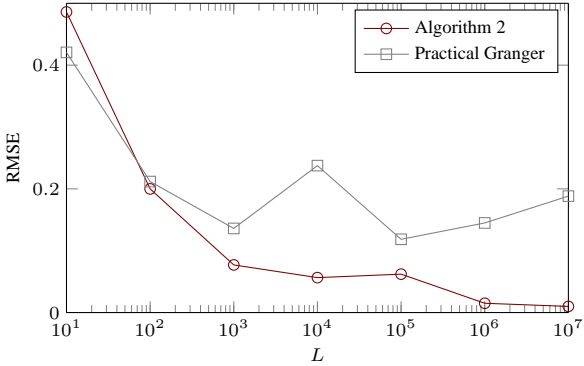

{\bf Discussion:} This empirically shows that the set of two outputs of Algorithm \ref{alg::cov_algo} asymptotically seem to contain the true $B$.
However, it takes at least 1000 samples to output reasonable estimates.
As expected, the practical Granger estimator does not seem to converge against the true $B$.

\subsection{Real-World Data}
Here we examine how Algorithm \ref{alg::vem_algo} performs on a real-world data set.

{\bf Experimental setup:} We consider a time series $Y$ of length 340 and the three components: cheese price $Y^1$, butter price $Y^2$, milk price $Y^3$, recorded monthly from January 1986 to April 2014\footnote{The data was retrieved from \url{http://future.aae.wisc.edu/tab/prices.html} on 29.05.2014.}.
We used the following estimators:
We applied practical Granger estimation to the full time series $Y$ (i.e.,\ considering $X = Y$) and denote the outcome by $A_{\text{fG}}$.
We applied practical Granger estimation to the reduced time series $(Y^1, Y^2)^\top$ (i.e.,\ considering $X = (Y^1, Y^2)^\top$) and denote the outcome by $B_{\text{pG}}$.
We applied Algorithm \ref{alg::vem_algo} to the full time series $Y$ (i.e.,\ considering $X = Y$), while assuming an additional hidden univariate $Z$, and denote the outcome by $\bar{A}_{\text{fA}}$.
We applied Algorithm \ref{alg::vem_algo} to the reduced time series $(Y^1, Y^2)^\top$ (i.e.,\ considering $X = (Y^1, Y^2)^\top$), while assuming an additional hidden univariate $Z$, and denote the outcome by $\tilde{A}_{\text{pA}}$.
Furthermore we do a model check as suggested in Section \ref{sec::model_checking}, although the sample size may be too small for the independence test to work reliably.

{\bf Outcome:} The outputs are:
\begin{align*}
 A_{\text{fG}} &= \left( \begin{array}{ccc} 0.8381 & 0.0810 & 0.0375 \\ 0.0184 & 0.9592 & -0.0473 \\ 0.2318 & 0.0522 & 0.7446 \end{array} \right) , \\
 B_{\text{pG}} &= \left( \begin{array}{cc} 0.8707 & 0.0837 \\ -0.0227 & 0.9559  \end{array} \right) , \\
 \bar{A}_{\text{fA}} &= \left( \begin{array}{cccc} 0.8809 & 0.1812 &  0.1016 & -0.1595 \\ 0.0221 & 1.0142 & -0.0290 & -0.0492 \\ 0.2296 & 0.1291 & 0.8172 & -0.1143 \\ 1.0761 & 0.6029 & -0.7184 & 0.4226 \end{array} \right) , \\
 \tilde{A}_{\text{pA}} &= \left( \begin{array}{ccc} 0.9166 & 0.0513 & -0.0067 \\ -0.0094 & 0.9828 & -0.0047 \\ -0.0031 & 0.1441 & -0.2365 \end{array} \right) .
\end{align*}
The outcome of the model check, based on a significance level of 5\%, is the following: the hypothesis of Gaussianity is rejected. Also the independence hypothesis stated in Section \ref{sec::model_checking} is rejected. The latter implies that the model assumptions underlying Algorithm \ref{alg::vem_algo} are probably wrong.

{\bf Discussion:} 
We consider $A_{\text{fG}}$ as ground truth. Intuitively, non-zero entries at positions $(i, 3)$ can be explained by the milk price influencing cheese/butter prices via production costs, while non-zero entries at positions $(3, j)$ can be explained by cheese/butter prices driving the milk price via demand for milk. The explanation of non-zero entries at positions (1, 2) an (2, 1) is less clear.
One can see that the upper left $2 \times 2$ submatrix  of $\tilde{A}_{\text{pA}}$ is quite close to that of $A_{\text{fG}}$ (the RMSE over all entries is 0.0753), which shows that Algorithm \ref{alg::vem_algo} works well in this respect. 
Note that $B_{\text{pG}}$ is even a bit closer (the RMSE is 0.0662).
However, the upper right $2 \times 1$ matrix of $\tilde{A}_{\text{pA}}$ is not close to a scaled version of the upper right $2 \times 1$ submatrix of $A_{\text{fG}}$ (which corresponds to $C$).
This is in contrast to what one could expect based on Theorem \ref{thm::id_C}.
$\bar{A}_{\text{fA}}$ can be seen as an alternative ground truth. 
It is important to mention that the estimated order (lag length) of the full time series $Y$ is 3, according to Schwarz's criterion (SC) \citep{Luetkepohl2006}, which would violate our assumption of a VAR process of order 1 (Section \ref{sec::model}).
The model check seems to detect this violation of the model assumptions.

\section{Conclusions}
\label{sec::conclusion}

We considered the problem of causal inference from observational time series data.
Our approach consisted of two parts: 
First, we examined possible conditions for identifiability of causal properties of the underlying system from the given data.
Second, we proposed two estimation algorithms and showed that they work on simulated data under the respective conditions from the first part.

\section*{Acknowledgements}

Kun Zhang was supported in part by DARPA grant No. W911NF-12-1-0034.

\appendix
\onecolumn
\section*{Appendices}

\section{Proofs for Section \ref{sec::genres}}
\label{sec::genres_proofs}

For this section keep in mind the definitions of $W,X,Z, N, N^X, N^Z$ and $A, B, C, D, E$ from Section \ref{sec::model} as well as $M_1$ from Section \ref{sec::genres}.

%

\begin{proof}[Proof of Lemma \ref{lem::repr_R}]
The case $K = K_X$ is obvious, so we only prove the case $K > K_X$.

In particular, keep in mind that 
\begin{align*}
\left( \begin{array}{c} X_t \\ Z_t \end{array} \right) &= A \left( \begin{array}{c} X_{t-1} \\ Z_{t-1} \end{array} \right) + N_{t}  ,
\end{align*}
and
\begin{align}
 A = \left( \begin{array}{cc} B & C \\ D & E \end{array} \right) .
\end{align}

Hence based on
\begin{align*}
\left( \begin{array}{c} X_t \\ Z_t \end{array} \right) &= A^2 \left( \begin{array}{c} X_{t-2} \\ Z_{t-2} \end{array} \right) + A N_{t-1} + N_{t} ,
\end{align*}
and
\begin{align*}
A^2 &= \left( \begin{array}{cc} B^2 + C D & B C + C E \\ D B + E D & D C + E^2 \end{array} \right) .
\end{align*}
we get
\begin{align}
X_t &= (B^2 + C D) X_{t-2} + (B C + C E) Z_{t-2} + B N^X_{t-1} + C N^Z_{t-1} + N^X_{t} , \label{eqn::X_t} \\
X_{t-1} &= B X_{t-2} + C Z_{t-2} + N^X_{t-1} . \label{eqn::X_t-1}
\end{align}

Based on the definition of the generalized residual $R_t(U_1,U_2)$ in Section \ref{sec::genres} and equations (\ref{eqn::X_t}) and (\ref{eqn::X_t-1}), we have
\begin{align}
 &R_t(U_1,U_2) \notag \\
 &= X_t - U_1 X_{t-1} - U_2 X_{t-2} \notag \\ 
 &= (B^2 + C D) X_{t-2} + (B C + C E) Z_{t-2} + B N^X_{t-1} + C N^Z_{t-1} + N^X_{t} \notag \\
 &\phantom{=} - U_1 (B X_{t-2} + C Z_{t-2} + N^X_{t-1}) - U_2 X_{t-2} \notag \\
 &= (B^2 + C D - U_1 B - U_2 ) X_{t-2} + (B C + C E - U_1 C) Z_{t-2} \notag \\
 &\phantom{=} + (B - U_1) N^X_{t-1} + C N^Z_{t-1} + N^X_{t} . \notag
\end{align}

\end{proof}


\begin{proof}[Proof of Lemma \ref{lem::zero_impl}]

Equation (\ref{eqn::coeff_zero}) together with equation (\ref{eqn::gres}) implies
\begin{align*}
 &R_t(U_1,U_2) = (B - U_1) N^X_{t-1} + C N^Z_{t-1} + N^X_{t} .
\end{align*}

Based on $\|A\| < 1$, we have \citep{Luetkepohl2006}
\begin{align*}
\left( \begin{array}{c}  X_t  \\ Z_t \end{array} \right) = W_t = \sum_{i=0}^\infty A^i N_{t-i} = \sum_{i=0}^\infty A^i \left( \begin{array}{c}  N^X_{t-i}  \\ N^Z_{t-i} \end{array} \right) . 
\end{align*}

This implies that
\[ ( X_{t-2-j} )_{j=0}^\infty \ind N^X_{t-1}, N^Z_{t-1}, N^X_{t} .\]
\end{proof}



\begin{proof}[Proof of Lemma \ref{lem::uncorrimpl}]
Keep in mind that 
\begin{align*}
 M_1 &= \E \left[ \left( \begin{array}{c} X_{t}   \\  Z_{t} \end{array} \right) ( X_t^T , X_{t-1}^T ) \right] \\
&= \left( \begin{array}{cc} \E[ X_{t} X_{t}^T] & \E[ X_{t} X_{t-1}^T]  \\ \E[ Z_{t} X_{t}^T] & \E[ Z_{t} X_{t-1}^T] \end{array} \right) .\\
 \end{align*}

Based on equation (\ref{eqn::gres}), we have for $j=0,1$
\begin{align}
0 &= \cov(R_t(U_1,U_2) , X_{t-2-j}) \notag \\
&= (B^2 + C D - U_1 B - U_2 ) \cov( X_{t-2}, X_{t-2-j})  \\
&\phantom{=} + (B C + C E - U_1 C) \cov(Z_{t-2}, X_{t-2-j}) \notag \\
&= (B^2 + C D - U_1 B - U_2 ) \E[ X_{t} X_{t-j}^T]  + (B C + C E - U_1 C) \E[ Z_{t} X_{t-j}^T] . \label{eqn::l2_1}
\end{align}


We can write equation (\ref{eqn::l2_1}) as the following system of linear equations
\begin{align*}
 \left( B^2 + C D - U_1 B - U_2 , B C + C E - U_1 C \right) \left( \begin{array}{cc} \E[ X_{t} X_{t}^T] & \E[ X_{t} X_{t-1}^T]  \\ \E[ Z_{t} X_{t}^T] & \E[ Z_{t} X_{t-1}^T] \end{array} \right) = 0,
\end{align*}
that is 
\begin{align*}
 \left( B^2 + C D - U_1 B - U_2 , B C + C E - U_1 C \right) M_1 = 0 .
\end{align*}

Since we assumed that $M_1$ has full rank, we can conclude
\begin{align*}
B^2 + C D - U_1 B - U_2 = 0 \quad\wedge\quad B C + C E - U_1 C = 0 . 
\end{align*}

\end{proof}



\begin{proof}[Proof of Lemma \ref{lem::U_exist}]


$C$ is a $K_X \times K_Z$ matrix of full rank, with $K_Z \leq K_X$, hence $C$ has full row rank. Hence $\left( \begin{array}{cc}  B &  C  \\ \I & 0 \end{array} \right)$ has full row rank. 
Thus, there is a $(U_1,U_2)$ which solves equation (\ref{eqn::coeff_zero}).

\end{proof}

\section{Proofs for Sections \ref{sec::idr} and \ref{sec::id_D_zero}}
\label{sec::id_proofs}

Recall assumptions A1, A2, A3, G1, G2 and the definition of $F_1, F_2$ in Sections \ref{sec::idr} and \ref{sec::id_D_zero} and the definition of $W,X,Z$ and $A,B,C,D,E$ from Section \ref{sec::model}.


\subsection{Proof of Theorem \ref{thm::idr}}
\label{sec::thm1}

Keep in mind that by a \emph{representation} of a random vector $Y$ we mean a matrix $Q$ together with a random vector $F = (f_1, \ldots, f_r)$ with independent components, such that $Y = Q F$.


To prove Theorem \ref{thm::idr} we need the following seminal result which is contained in \citep[Theorem 10.3.1]{Kagan73}.
It allows to exploit non-Gaussianity of noise terms to achieve a certain kind of identifiability.
The theorem will be at the core of the proof of Theorem \ref{thm::idr}.

\begin{Theorem}
\label{thm::overcomp_ica}
 Let $Y = Q F$ and $Y = R G$ be two representations of a $p$-dimensional random vector, where $Q$ and $R$ are constant matrices of order $p \times r$ and $p \times s$ respectively, and $F = (f_1, \ldots, f_r)$ and $G = (g_1, \ldots, g_s)$ are random vectors with independent components. Then the following assertion holds.
 If the $i$-th column of $Q$ is not proportional to any column of $R$, then $F_i$ is normal.
\end{Theorem}

%

We proceed with the proof of Theorem \ref{thm::idr}.

\begin{proof}[Proof of Theorem \ref{thm::idr}]
\stress{Ansatz:}

We prove that given $P_X$, the structural matrix $B$ underlying $X$ is determined uniquely. 

\stress{Choosing $(U_1,U_2)$:}

Based on assumption G1 and Lemmas \ref{lem::U_exist} and \ref{lem::zero_impl}, there always exists $(U_1,U_2)$ such that
 \begin{align}
\cov(R_t(U_1,U_2) , X_{t-2-j}) = 0 .
\end{align}
Pick one such $(U_1,U_2)$.

\stress{Deriving a representation for $\left( \begin{array}{c}  R_t(U_1,U_2)  \\ R_{t-1}(U_1,U_2) \end{array} \right)$:}


Based on Lemma \ref{lem::uncorrimpl}, we know that 
\begin{align*}
B^2 + C D - U_1 B - U_2 = 0 \quad\wedge\quad B C + C E - U_1 C = 0 ,
\end{align*}
and thus, based on equation (\ref{eqn::gres}),
\begin{align*}
 R_t(U_1,U_2) = N^X_{t} +  C N^Z_{t-1} +  (B - U_1) N^X_{t-1} .
\end{align*}

Observe that

\begin{align*}
 &\left( \begin{array}{c}  R_t(U_1,U_2)  \\ R_{t-1}(U_1,U_2) \end{array} \right) \\
 &= \left( \begin{array}{ccc}  N^X_{t} +  C N^Z_{t-1} + & (B - U_1)  N^X_{t-1} &  \\ & N^X_{t-1} & +  C N^Z_{t-2} +  (B - U_1) N^X_{t-2} \end{array} \right) \\
 &= \left( \begin{array}{ccccc}  \I & C & (B - U_1) & 0 & 0  \\ 0 & 0 & \I & C & (B - U_1) \end{array} \right) \left( \begin{array}{c} N^X_{t} \\  N^Z_{t-1} \\  N^X_{t-1}  \\  N^X_{t-2}   \\  N^Z_{t-2}  \end{array} \right) \\
 &=: Q \tilde{N}_t .
\end{align*}

This is one representation of $\left( \begin{array}{c}  R_t(U_1,U_2)  \\ R_{t-1}(U_1,U_2) \end{array} \right)$. 

Based on Theorem \ref{thm::overcomp_ica} and the structure of $Q$, $B - U_1$ is identifiable from $\left( \begin{array}{c}  R_t(U_1,U_2)  \\ R_{t-1}(U_1,U_2) \end{array} \right)$.
This can be seen as follows. 

\stress{Identifying $B - U_1$ from $\left( \begin{array}{c}  R_t(U_1,U_2)  \\ R_{t-1}(U_1,U_2) \end{array} \right)$:}

Knowing $P_X$, we also know $P_{(R_t(U_1,U_2), R_{t-1}(U_1,U_2))}$ which in particular determines the class of all possible representations of $\left( \begin{array}{c}  R_t(U_1,U_2)  \\ R_{t-1}(U_1,U_2) \end{array} \right)$.
Pick one representation $\left( \begin{array}{c}  R_t(U_1,U_2)  \\ R_{t-1}(U_1,U_2) \end{array} \right) = Q' \tilde{N}_t'$ out of this class.
W.l.o.g. let $Q'$ be such that all its columns are pairwise linearly independent.

Theorem \ref{thm::overcomp_ica} implies that each column of $Q'$ is a scaled version of some column of $Q$ and vice versa.

Now define the $K_X \times K_X$ matrix $V:=(v_1, \ldots, v_{K_X})$ as follows.

For each $j=1, \ldots, K_X$: 

If $Q'$ has a column with a non-zero entry at position $K_X + j$ and a non-zero entry in the upper half, let this column be denoted by $q_j$ and define
\[ v_j := \left[ \frac{1}{[q_j]_{K_X + j}}  q_j \right]_{1:K_X} , \]
where $[ q ]_{k_1, \ldots, k_l}$ denotes the $l$-dimensional vector consisting of $k_1$st to $k_l$th entry of a vector $q$, and $k:l$ is shorthand for $k, k+1, \ldots, l$.
Otherwise, if $Q$ has no such column, then set
\[ v_j := 0. \]

We have $V = B - U_1$. This can be seen as follows:

Let $w_j$ denote the $j$th column of $B - U_1$.

For each $j=1, \ldots, K_X$:

Either we have $w_j \neq 0$. Then the corresponding column in $Q$, i.e. $\left( \begin{array}{c}  w_j  \\ e_j \end{array} \right)$, where $e_j$ denotes the $j$th unit vector, is the only column with a non-zero entry at position $K_X + j$ and a non-zero entry in the upper half. 
Thus $Q'$ contains a scaled version of $\left( \begin{array}{c}  w_j  \\ e_j \end{array} \right)$ and no other column with a non-zero entry at position $K_X + j$ and a non-zero entry in the upper half.
We denoted this column by $q_j$ and defined $v_j = \left[ \frac{1}{[q_j]_{K_X + j}}  q_j \right]_{1:K_X}$.
Since $\left[ \frac{1}{[q_j]_{K_X + j}}  q_j \right]_{K_X + j} = 1 = \left[ \left( \begin{array}{c}  w_j  \\ e_j \end{array} \right) \right]_{K_X + j}$, we know that $\frac{1}{[q_j]_{K_X + j}}  q_j = \left( \begin{array}{c}  w_j  \\ e_j \end{array} \right)$ and hence $v_j = w_j$.

Or we have $w_j = 0$. Then $Q$ and hence also $Q'$ contains no column with a non-zero entry at position $K_X + j$ and a non-zero entry in the upper half.
Then by definition we have $v_j = 0$ and thus again $v_j = w_j$.

Hence $V = B - U_1$.

\stress{Putting all together:}

We defined $U_1$ solely based on $P_X$ and an arbitrary choice and then, for the fixed $U_1$, uniquely determined $B-U_1$, again only based on $P_X$. 
Hence $B = U_1 + (B - U_1)$ is uniquely determined by $P_X$.

\end{proof}

\subsection{Proof of Theorem \ref{thm::id_C}}
\label{sec::thm2}

%
%

Here we prove Theorem \ref{thm::id_C}.

\begin{proof}[Proof of Theorem \ref{thm::id_C}]
Keep in mind the proof of Theorem \ref{thm::idr}. 
There we showed that the matrix
\begin{align*}
 Q = \left( \begin{array}{ccccc}  \I & C & (B - U_1) & 0 & 0  \\ 0 & 0 & \I & C & (B - U_1) \end{array} \right) 
\end{align*}
is identifiable from $P_X$ up to scaling and permutation of its columns, for some $U_1$.
This implies that we can identify the matrix
\begin{align*}
 Q_1 = \left( \begin{array}{cc}  \I & C  \end{array} \right) 
\end{align*}
up to scaling and permutation of its columns, simply by picking those columns of any scaled and permuted version of $Q_1$, that only have non-zero entries in the upper half.

But this in turn implies that we can identify the set of columns of $C$ with at least two non-zero entries up to scaling of those columns.
Just pick from any scaled and permuted version of $Q_1$ those columns with at least two non-zero entries.

\end{proof}

\subsection{Proof of Theorem \ref{thm::almost_id}}
\label{sec::thm3}


Following standard terminology \citep{Dennis1976}, for any $n \times n$-matrices $F_1, \ldots, F_m$ and $Y$ we call 
\[ M(Y) := F_0 Y^m + F_1 Y^{m-1} + \ldots + F_m \]
a \defi{matrix polynomial of degree $m$}.
We say a matrix $Y_0$ is a \defi{right solvent} or simply \defi{solvent} of $M(Y)$, if $M(Y_0) = 0$.
We say $\lambda \in \C$ is a \defi{latent root} of $M(Y)$, if, slightly overloading notation, $M(\lambda) := M(\lambda \I)$ is not invertible.

To prove Theorem \ref{thm::almost_id} we need the following result which is a version of \citep[Corollary 4.1]{Dennis1976}.

\begin{Theorem}
\label{thm::matrixpoly}
Let $M(Y) := F_0 Y^m + F_1 Y^{m-1} + \ldots + F_m$ be any matrix polynomial, where $F_1, \ldots, F_m$ are $n \times n$ square matrices.
If $M(\lambda)$ has $m n$ distinct latent roots, then it has at most $\binom{m n}{n}$ different right solvents.
\end{Theorem}

(Note that this assertion is also stated in the conclusion section of \citep{Pereira2003} but without proof it seems.)

\begin{proof}
In this proof we assume the paper \citep{Dennis1976} as context. 
That is, \emph{all definitions and equations we refer to in this proof are meant w.r.t. that paper}.

Let $S$ be a solvent of $M(Y)$.
By the corollary containing equation (1.4), we have $M(\lambda) = Q(\lambda) (\I \lambda - S)$, with $Q(\lambda)$ a matrix polynomial of degree $m-1$.
By assumption, we know that $\det(M(\lambda)) = \det( Q(\lambda) ) \det( \I \lambda - S)$ has $m n$ distinct roots. 
Since $\det( Q(\lambda) )$ has at most $(m-1)n$ different roots, we know that $\det( \I \lambda - S)$ has to have $n$ different roots. 
Hence $S$ has $n$ distinct eigenvalues and is uniquely determined by its $n$ eigenpairs, i.e.\ pairs $(a, \C v)$ such that $S v = a v$.


Keep in mind that a latent pair of $M(\lambda)$ is a scalar $a$ together with a ray $\C v$ for some vector $v \neq 0$ such that $M(a) v = 0$.
Let $L$ denote the set of latent pairs of $M(\lambda)$.
Based on equation (1.4), each eigenpair of a solvent $S$ is a latent pair of $M(\lambda)$.
Hence for each solvent $S$, the tuple of $n$ eigenpairs that uniquely determines this solvent has to be a subset of size $n$ of $L$. 
Therefore, the number of solvents is less or equal than $\binom{|L|}{n}$, in case $L$ is finite.

Consider the $m n \times m n$ matrix $C^B$ defined by equation (3.2). 
Theorem 3.1, applied to $C^B$ (see remark above equation (3.2)), states that 
\[\det(C^B - \lambda I ) = (-1)^{m n} \det( M(\lambda) ). \]
Hence $\det(C^B - \lambda I )$ has exactly $m n$ distinct roots.

Now assume that $|L| > m n$, i.e., $M(\lambda)$ has more than $m n$ latent pairs. 
Then there have to be two latent pairs $(a,\C v)$ and $(a,\C v')$ with $\C v \neq \C v'$.
Based on Theorem 3.2, part (i), this implies that $C^B$, as defined by equation (3.2), has two linearly independent vectors as eigenvectors to the same eigenvalue $a$.
Thus the eigenvalue $a$ has geometric and hence also algebraic multiplicity at least 2.
This implies that $\det(C^B - \lambda I )$ has exactly $m n$ distinct roots and at least one of the roots, namely $a$, has algebraic multiplicity at least 2.
This is a contradiction to the fact that $\det(C^B - \lambda I )$ has degree $m n$.

\end{proof}

%
%

\begin{proof}[Proof of Theorem \ref{thm::almost_id}]

Keep in mind that assumption A3 reads $D=0$.

Let $S_1$ denote the set of $U=(U_1,U_2)$ such that
 \begin{align}
\cov(R_t(U_1,U_2) , X_{t-2-j}) = 0 .
\end{align}

Let $S_2$ denote the set of $U=(U_1,U_2)$ such that $\det(T_U(\alpha))$ has $2 K_X$ distinct roots.

Based on the assumption G2, there exists $U=(U_1, U_2)$ such that the equation 
\begin{align}
 \left( U_1 , U_2 \right) \left( \begin{array}{cc}  B &  C  \\ \I & 0 \end{array} \right) = \left( B^2 , B C + C E \right) 
\end{align}
is satisfied and $\det(T_U(\alpha))$ has $2 K_X$ distinct roots. 
This $U$ is in $S_2$ and based on Lemma \ref{lem::zero_impl} it is also in $S_1$.
Hence $S := S_1 \cap S_2$ is non-empty.

Note that $S$ is defined only based on $P_X$.

Pick one $U=(U_1,U_2)$ out of $S$.

Let
\begin{align*}
 L := \{ \tilde{B} : T_{U}(\tilde{B}) = 0 \} .
\end{align*}

Based on Theorem \ref{thm::matrixpoly}, $L$ has at most $\binom{2 K_X}{K_X}$ elements.
And since $U \in S_1$, assumption G1 together with Lemma \ref{lem::uncorrimpl} implies that $B \in L$, for the true $B$.

Hence $B$ is determined by $P_X$ up to $\binom{2 K_X}{K_X}$ possibilities.

\end{proof}

\section{Discussion on the Genericity of Assumptions G1 and G2: An Elaboration of Section \ref{sec::genericity}}
\label{sec::genericity_proofs}


This section is an elaborated version, including proofs, of Section \ref{sec::genericity}.

We want to argue why the assumptions G1 and G2 stated in Sections \ref{sec::idr} and \ref{sec::id_D_zero} are generic.
Keep in mind the definitions of $W,X,Z, N, N^X, N^Z$ and $A, B, C, D, E, \Sigma$ from Section \ref{sec::model} as well as $M_1$ from Section \ref{sec::genres}.
The idea is to define a natural parametrization of $(A,\Sigma)$ and to show that the restrictions that assumptions G1 and G2, respectively, impose on $(A,\Sigma)$ just exclude a Lebesgue null set in the natural parameter space.

Have in mind that Theorems \ref{thm::idr} and \ref{thm::almost_id} state (almost) identifiability of $B$ from $P_X$ induced by any $W$ in $F_1$ and $F_2$, respectively. In particular, such $W$ can have \emph{arbitrary numbers of components $K$}, as long as $K_X \leq K \leq 2K_X$.
However, for the sake of simplicity, we show the genericity of assumptions G1 and G2 only under the assumption of an arbitrary but \emph{fixed} $K$.
Therefore, in this section, let $K$ such that $K_X \leq K \leq 2K_X$ be arbitrary but fixed.
As usual, let $K_Z = K - K_X$.

Let $\lambda_k$ denote the $k$-dimensional Lebesgue measure on $\R^k$.
Let $\vect$ denote the column stacking operator and $\vect^{-1}$ its inverse.
The dimension of the domain of $\vect$ can always be understood from the context.
For a vector $q$, let $[ q ]_{k_1, \ldots, k_l}$ denote the $l$-dimensional vector consisting of $k_1$st to $k_l$th entry of $q$.
Moreover, let $k:l$ be shorthand for $k, k+1, \ldots, l$.

\subsection{Assumption G1 in Theorems \ref{thm::idr} and \ref{thm::id_C}}
\label{sec::genericity1}

Let $\Theta_1$ denote the set of all possible parameters $(A',\Sigma')$ for a $K$-variate VAR processes $W'$ that additionally satisfy assumption A2, i.e., correspond to structural $W'$.
Let $S_1$ denote the subset of those $(A',\Sigma') \in \Theta_1$ for which also assumption G1 is satisfied.

(The relation between $S_1$ as defined above and $F_1$ as defined in Section \ref{sec::idr} is the following: for any process $W'$ with parameters $(A',\Sigma')$, $W' \in F_1$ iff $W'$ satisfies assumption $A1$ (i.e., its noise components are non-Gaussian) and additionally $(A',\Sigma') \in S_1$.)

To parametrize $\Theta_1$ in a practical way, let $g=(g_1,g_2): \R^{K^2 + K} \to \R^{K^2} \times \R^{K^2}$ be defined by
\begin{align*}
g_1(v) &:=  \vect^{-1}( [ v ]_{1:K^2} )  , \\
g_2(v) &:=  \diag( [ v ]_{K^2 + 1 : K^2 + K} ) ,
\end{align*}
for all $v \in \R^{K^2 + K}$.
Hence $g_1$ is the natural parametrization of $A$ and $g_2$ for $\Sigma$.

We repeat the proposition already stated in Section \ref{sec::genericity}:

\begin{p1}
We have $\lambda_{K^2 + K} \left( g^{-1}(\Theta_1 \setminus S_1) \right) = 0$.
\end{p1}

Let $\Phi_1 := g^{-1}(\Theta_1)$.
Since $g|_{\Phi_1} : \Phi_1 \to \Theta_1$ is a linear bijective function, the above statement can be interpreted as $\Theta_1 \setminus S_1$ being very small and thus G1 being a requirement that is met in the generic case.

\subsubsection{Proof of Proposition \ref{prop::genericity1}}

The proof idea for Proposition \ref{prop::genericity1} is that $g^{-1}(\Theta_1 \setminus S_1)$ is essentially contained in the union of the root sets of finitely many multivariate polynomials and hence is a Lebesgue null set.
Before we give a rigorous proof, we first need introduce some definitions and establish two lemmas.

\begin{Lemma}
\label{lem::poly_roots}
For any $n$ and any non-zero multivariate polynomial $q(x_1,\ldots,x_n)$, the set
\[ L := \{ (x_1,\ldots,x_n) \in \R^n : q(x_1,\ldots,x_n) = 0 \} \]
is a null set  w.r.t. the $n$-dimensional Lebesgue measure on $\R^n$.
\end{Lemma}

\begin{proof}
We prove the statement via induction over $n$.

\stress{Basis:}
 
Let $n=1$. 
Let $q(x_1)$ be any non-zero polynomial.
By the fundamental theorem of algebra it follows immediately that it has at most $\deg(q)$ real roots. 
Hence $L$ is a Lebesgue null set.

\stress{Inductive step:}

Now assume the statement holds for all multivariate polynomials in less than $n$ variables.
Let  $q(x_1,\ldots,x_n)$ be any $n$-variate non-zero polynomial.
We can consider $q$ as a univariate polynomial in $x_1$, denoted by $r(x_1;x_{2:n})$, with coefficients $r_i(x_{2:n})$ that are multivariate polynomials in $x_{2:n}$, i.e.
\[ q(x_1,\ldots,x_n) = r(x_1;x_{2:n}) = r_0(x_{2:n}) + r_1(x_{2:n}) x_1 + \ldots + r_l(x_{2:n}) x_1^l ,\]
for some $l$.

There has to be some $j$ such that $r_j(x_{2:n})$ is not the zero polynomial, since otherwise $q(x_1,\ldots,x_n)$ would be the zero polynomial.
Let
\[ L' := \{ (x_2,\ldots,x_n) \in \R^{n-1} : r_j(x_2,\ldots,x_n) = 0 \} .\]
By induction, we know that $\lambda_{n-1}(L') = 0$.
Hence $r(x_1;x_{2:n})$ is a non-zero polynomial for all $x_{2:n} \in \R^{n-1} \setminus L'$.
In particular, due to the fundamental theorem of algebra, for all $x_{2:n} \in \R^{n-1} \setminus L'$, the set $L_{x_{2:n}} := \{ x_1 \in \R : r(x_1;x_{2:n})=0 \}$ is finite (has at most $n-1$ elements). 

Note that, since $q$ is continuous, $L = q^{-1}(\{ 0 \})$ is closed and thus measurable.
Let $\1$ denote the indicator function.
In particular, $\1_L$ is measurable. 
Furthermore, note that $\1_L(x_{1:n}) = \1_{L_{x_{2:n}}}(x_1)$ for all $x_{1:n}$.
Therefore and due to Fubini's theorem (for completed product spaces) we have
\begin{align*}
 \lambda_n(L) &= \int_{\R^n} \1_L(x_1,\ldots,x_n) \ \td x_{1:n} \\
 &= \int_{\R^{n-1}} \int_\R \1_{L_{x_{2:n}}}(x_1) \ \td x_1\ \td x_{2:n} \\ 
 &= \int_{\R^{n-1} \setminus L'} \int_\R \1_{L_{x_{2:n}}}(x_1) \ \td x_1\ \td x_{2:n} \\ 
 &= \int_{\R^{n-1} \setminus L'}  \lambda_1(L_{x_{2:n}}) \ \td x_{2:n} \\ 
 &= \int_{\R^{n-1} \setminus L'}  0 \ \td x_{2:n} \\ 
 &= 0.
\end{align*}

\end{proof}

Let $\Psi_1 := g^{-1}(S_1)$.

For a $I \times J$ matrix 
\[ M= \left( \begin{array}{ccc} m_{1 1} & \ldots & m_{1 J} \\  & \vdots & \\ m_{I 1} & \ldots & m_{I J}   \end{array} \right) ,  \]
let $[M]_{ij} := m_{i j}$ and $[M]_{i_1:i_2, j_1:j_2} := (m_{i j})_{i_1 \leq i \leq i_2, j_1 \leq j \leq j_2 } $.

Keep in mind the following equations for the autocovariance matrices $\Gamma_i := \E[ \tilde{W}_t \tilde{W}_{t-i}^\top ]$ of any VAR process $\tilde{W}$ with parameters $(\tilde{A},\tilde{\Sigma})$ \citep{Luetkepohl2006}:
\begin{align}
 \vect(\Gamma_0) &= ( \I - \tilde{A} \otimes \tilde{A} )^{-1} \vect( \tilde{\Sigma} ) , \label{eqn::l1}\\
 \Gamma_i &= \tilde{A}^i \Gamma_{i-1} . \label{eqn::l2}
\end{align}

In this subsection, given any $\phi \in \Phi_1$, let $W^\phi$ be some $K$-variate VAR process with parameters $g(\phi)$, and let $X^\phi$ denote the first $K_X$ and $Z^\phi$ denote the remaining $K - K_X$ components of $W^\phi$.

And also for this subsection, for any $\phi \in \Phi_1$ and $i \geq 0$, let $\Gamma_i(\phi) := \E[ W^\phi_t (W^\phi)_{t-i}^\top ]$.

Recall the definition of $M_1$ from Section \ref{sec::genres}. Here we explicitly consider $M_1$ as a function on $\Phi_1$. That is, for any $\phi \in \Phi_1$ let
\begin{align*}
 M_1(\phi) := \E \left[ W^\phi_t ( ( X^\phi_t )^\top , ( X^\phi_{t-1} )^\top ) \right] .
\end{align*}
Later we want to shot that the set of $\phi \in \Phi_1$ for which $M_1(\phi)$ does not have full rank is a Lebesgue null set. 
It suffices to show that $M_1(\phi)$ has a fixed square submatrix $M_2(\phi)$ such that the set of $\phi \in \Phi_1$ for which $M_2(\phi)$ is not invertible is a Lebesgue null set, since the former set is contained in the latter.
For this purpose let us define
 \begin{align*}
 M_2(\phi) := \left\{ \begin{array}{ll} \E \left[ W^\phi_t \left( ( X^\phi_t  )^\top ,  [ X^\phi_{t-1} ]_{K_X - K_Z: K_X}  \right) \right] , & \text{if } K > K_X   \\  \E \left[ W^\phi_t  ( W^\phi_t  )^\top  \right] ( = \Gamma_0(\phi) ), & \text{if } K = K_X \end{array} \right. .
 \end{align*}
 
That is, $M_2$ is a $K \times K$ square matrix with a subset of the columns of $M_1$ as columns 
(keep in mind that $[ X^\phi_{t-1} ]_{K_X - K_Z: K_X}$ are the $(K_X - K_Z)$-th to $K_X$-th components of $X^\phi_{t-1}$).

Let 
\begin{align}
 f(\phi) := \det(M_2(\phi)) . \label{eqn::f_defi}
\end{align}

\begin{Lemma}
\label{lem::theta_ex}
 There is some $\phi \in \Phi_1$ such that $f(\phi) \neq 0$.
\end{Lemma}

\begin{proof}

We only treat the cases $K = K_X$ and $K = K_X + 1$. The cases $K_X + 1 < K \leq 2 K_X$ can be treated similarly. 

\stress{The case $K = K_X$:}

Let $\tilde{A} := \frac{1}{2} \I$ and $\tilde{\Sigma} := \I$ and let $\phi := g^{-1}(\tilde{A}, \tilde{\Sigma})$.
Based on equation (\ref{eqn::l1}) this immediately implies
\begin{align}
 M_2(\phi) = \Gamma_0(\phi) = \frac{4}{3} \I ,
\end{align}
and hence $f(\phi) = \det(M_2(\phi)) \neq 0$.

\stress{The case $K = K_X + 1$:}

Let $\tilde{\Sigma} := \I$ and
\begin{align*}
 \tilde{A} := \left( \begin{array}{ccc|cc} \frac{1}{2} & & & & \\ & \ddots & & & \\ & & \frac{1}{2} & & \\ \hline & & & \frac{1}{2} & \frac{1}{2} \\ & & & & \frac{1}{2} \end{array} \right) =:  \left( \begin{array}{c|c} \tilde{A}_1 & 0  \\ \hline 0 & \tilde{A}_2 \end{array} \right) ,
\end{align*}
and let $\phi := g^{-1}(\tilde{A}, \tilde{\Sigma})$ denote the corresponding parameter vector.
Now we want to calculate $\Gamma_0(\phi), \Gamma_1(\phi)$. For this purpose, observe that we can split $W^\phi$ into the two independent VAR processes
\begin{align*}
 Y^1 &:= ( [ X^\phi ]_1, \ldots, [ X^\phi ]_{K_X - 1} )^\top ,\\
 Y^2 &:= ( [ X^\phi ]_{K_X}, Z)^\top .
\end{align*}

Equation (\ref{eqn::l1}) applied to $Y^1$ implies
\[ \vect( \E[ Y^1_t ( Y^1_t)^\top ] ) = ( \I - \tilde{A}_1 \otimes \tilde{A}_1 )^{-1} \vect( \I ) = \frac{4}{3} \vect( \I ) , \]
that is
\[ \E[ Y^1_t (Y^1_t)^\top ] = \frac{4}{3} \I . \]

On the other hand, equation (\ref{eqn::l1}) applied to $Y^2$ yields
\begin{align*}
 \vect( \E[ Y^2_t ( Y^2_t )^\top ] ) = ( \I - \tilde{A}_2 \otimes \tilde{A}_2 )^{-1} \vect( \I ) = \frac{4}{27} \left( \begin{array}{cccc} 9 & 3 & 3 & 5 \\ & 9 & & 3 \\ & & 9 & 3 \\ & & & 9 \end{array} \right) \vect( \I ) ,
\end{align*}
that is
\begin{align*}
  \E[ Y^2_t ( Y^2_t)^\top ]  =  \frac{4}{27} \left( \begin{array}{cc} 14 & 3  \\  3 & 9 \end{array} \right) .
\end{align*}

Thus
\begin{align*}
 \Gamma_0(\phi) &= \E[ W^\phi_t (W^\phi)_{t}^\top ] = \left( \begin{array}{cc} \E[ Y^1_t ( Y^1_t )^\top ] & 0  \\  0 & \E[ Y^2_t ( Y^2_t )^\top ] \end{array} \right) \\ 
 &= \left( \begin{array}{ccccc} \frac{4}{3} & & & & \\ & \ddots & & & \\ & & \frac{4}{3} & & \\ & & &  \frac{56}{27} & \frac{4}{9} \\ & & & \frac{4}{9} & \frac{4}{3} \end{array} \right) ,
\end{align*}
and
\begin{align*}
 \Gamma_1(\phi) = \tilde{A} \Gamma_0(\phi) 
 = \left( \begin{array}{ccccc} \frac{2}{3} & & & & \\ & \ddots & & & \\ & & \frac{2}{3} & & \\ & & &  \frac{34}{27} & \frac{8}{9} \\ & & & \frac{2}{9} & \frac{2}{3} \end{array} \right) .
\end{align*}

Hence
\begin{align*}
 M_2(\phi) = \left( \begin{array}{ccccc} \frac{4}{3} & & & & \\ & \ddots & & &  \\ & &  \frac{4}{3} & & \\ & & & \frac{56}{27} & \frac{34}{27} \\ & & & \frac{4}{9} & \frac{2}{9} \end{array} \right) .
\end{align*}

Hence $\phi$ is such that $f(\phi) = \det(M_2(\phi)) \neq 0$.

\end{proof}

\begin{proof}[Proof of Proposition \ref{prop::genericity1}]

Recall that $\Phi_1 = g^{-1}(\Theta_1)$, $\Psi_1 = g^{-1}(S_1)$, $f(\phi) = \det(M_2(\phi))$, and how $S_1$ is related to $f$.

\stress{First, show that $f$ is a rational function:}

Keep in mind that each entry of $g_1(\phi)$ is a linear function in $\phi$.

For any $\phi \in \Phi_1$, let $G(\phi) := \I - g_1(\phi) \otimes g_1(\phi)$. Note that each entry of $G(\phi)$ is a multivariate polynomial in $\phi$.
We have for $i=0,1$, and for all $\phi \in \Phi_1$, using equation (\ref{eqn::l1}) and Cramer's rule,
\begin{align*}
 \Gamma_i(\phi) &= g_1(\phi)^i \vect^{-1}( G(\phi)^{-1} \vect( g_2(\phi)) ) \\
 &= g_1(\phi)^i \vect^{-1}( \det(G(\phi))^{-1} \adj(G(\phi)) \vect( g_2(\phi)) ) \\
 &= \det(G(\phi))^{-1} g_1(\phi)^i \vect^{-1}( \adj(G(\phi)) \vect( g_2(\phi)) ) .
\end{align*}
(Note that the definition of $\Phi_1$ implies that $\| g_1 \| < 0$ and thus $\det(G) \neq 0$ on $\Phi_1$.)

Keep in mind that for any matrix $Q$, the determinant $\det(Q)$ as well as all entries of the adjugate $\adj(Q)$, are multivariate polynomials in the entries of $Q$.
In particular each entry of $g_1(\phi)^i \vect^{-1}( \adj(G(\phi)) \vect( g_2(\phi)) )$ is a multivariate polynomial in $\phi$.

Now observe that on $\Phi_1$ we have
\begin{align*}
 f &= \det(M_2) \\
 &= \det\left( \left( \left[ \Gamma_0 \right]_{1:K, 1:K_X} , \left[ \Gamma_1 \right]_{1:K, K_X - K_Z: K_X} \right)\right) \\
 &= \det\left( \left( \left[ \det(G)^{-1} \vect^{-1}( \adj(G) \vect( g_2) ) \right]_{1:K, 1:K_X} , \left[ \det(G)^{-1} g_1 \vect^{-1}( \adj(G) \vect( g_2) ) \right]_{1:K, K_X - K_Z: K_X} \right)\right) \\
 &= \det(G)^{-K} \det\left( \left( \left[  \vect^{-1}( \adj(G) \vect( g_2) ) \right]_{1:K, 1:K_X} , \left[ g_1 \vect^{-1}( \adj(G) \vect( g_2) ) \right]_{1:K, K_X - K_Z: K_X} \right)\right)
\end{align*}

For all $\phi \in \R^{K^2 + K}$, let 
\begin{align*}
 r(\phi) &:= \det(G(\phi))^{K} , \\
 q(\phi) &:= \det\left( \left( \left[  \vect^{-1}( \adj(G) \vect( g_2) ) \right]_{1:K, 1:K_X} , \left[ g_1 \vect^{-1}( \adj(G) \vect( g_2) ) \right]_{1:K, K_X - K_Z: K_X} \right)\right) .
\end{align*}

Based on the above argument, $q(\phi),r(\phi)$ are multivariate polynomials (mappings from $\R^{K^2 + K}$ to $\R$).
Hence in particular, $f = \frac{q}{r}$ is a rational function on $\Phi_1$.

\stress{Second, show that $\lambda_{K^2 + K} \circ g^{-1}(\Theta_1 \setminus S_1) = 0$:}

In what follows, we only discuss the case $K > K_X$. The case $K = K_X$ works similarly and is even simpler.

Let $\tilde{C}(\phi)$ denote the upper right submatrix of $g_1(\phi)$ of dimension $K_X \times K_Z$.
Keep in mind that $\Psi_1 = g^{-1}(S_1)$ is the set of those $\phi \in \Phi_1 = g^{-1}(\Theta_1)$, for which $\tilde{C}(\phi)$ and $M_1(\phi)$ have full rank.

Let $H$ denote the set of those $\phi \in \Phi_1$, for which $\det\left( \left[ \tilde{C}(\phi) \right]_{1:K_Z, 1:K_Z} \right) = 0$.
Since $\det\left( \left[ \tilde{C}(\phi) \right]_{1:K_Z, 1:K_Z} \right)$ is a non-zero multivariate polynomial in $\phi$, based on Lemma \ref{lem::poly_roots} we have $\lambda_{K^2 + K}(H) = 0$.

Let $H'$ denote the set of those $\phi \in \Phi_1$, for which $q(\phi) = 0$.
Based on Lemma \ref{lem::theta_ex} we know that there is some $\phi$ such that $q(\phi) \neq 0$.
Hence based on Lemma \ref{lem::poly_roots} we have $\lambda_{K^2 + K}(H') = 0$.

If any $\phi$ is in $\Phi_1$ but neither in $H$ nor in $H'$, then $\det\left( \left[ \tilde{C}(\phi) \right]_{1:K_Z, 1:K_Z} \right) \neq 0$ and $q(\phi) \neq 0$, and thus $\tilde{C}(\phi)$ and $M_1(\phi)$ have full rank.
That is, $H^C \cap (H')^C \cap \Phi_1 \subset \Psi_1$. Therefore 
\begin{align*}
 \lambda_{K^2 + K} \left( g^{-1}(\Theta_1 \setminus S_1) \right) &= \lambda_{K^2 + K} \left( g^{-1}(\Theta_1) \setminus g^{-1}(S_1) \right) \\
 &= \lambda_{K^2 + K} \left( \Phi_1 \setminus \Psi_1 \right) \\
 &\leq \lambda_{K^2 + K} \left( \Phi_1 \setminus (H^C \cap (H')^C \cap \Phi_1) \right) \\
 &\leq \lambda_{K^2 + K} \left( \Phi_1 \setminus (H^C \cap (H')^C) \right) \\
 &= \lambda_{K^2 + K} \left( \Phi_1 \setminus (H \cup H')^C \right) \\
 &= \lambda_{K^2 + K} \left( \Phi_1 \cap (H \cup H') \right) = 0 .
\end{align*}

\end{proof}

\subsection{Assumptions G1 and G2 in Theorem \ref{thm::almost_id}}
\label{sec::genericity2}

Let $\Theta_2$ denote the set of all possible parameters $(A',\Sigma')$ for the $K$-variate VAR processes $W$ that additionally satisfy assumption A3, i.e., are such that the submatrix $D$ of $A$ is zero.
Let $S_2$ denote the subset of those $(A',\Sigma') \in \Theta_2$ for which also assumption G1 and G2 is satisfied.

To parametrize $\Theta_2$ in a practical way, let $h=(h_1,h_2): \R^{2 K^2 - K_X K_Z } \to \R^{K^2} \times \R^{K^2}$ be defined by
\begin{align*}
h_1(v) &:=  \left( \begin{array}{cc} \vect^{-1}( [ v ]_{1:K_X^2} ) & \vect^{-1}( [ v ]_\alpha )   \\  0 & \vect^{-1}( [ v ]_\beta ) \end{array} \right) , \\
h_2(v) &:=  \vect^{-1}( [ v ]_{K^2 - K_X K_Z + 1 : 2 K^2 - K_X K_Z} ) ,
\end{align*}
for all $v \in \R^{K^2 + K}$, where 
\begin{align*}
\alpha &:= {K_X^2 + 1 : K_X^2 + K_X K_Z} ,\\
\beta &:= {K_X^2 + K_X K_Z + 1 : K^2 - K_X K_Z} .
\end{align*}
Hence $h_1$ is the natural parametrization of $A$ and $h_2$ for $\Sigma$.

We repeat the proposition already stated in Section \ref{sec::genericity}:

\begin{p2}
We have $\lambda_{2 K^2 - K_X K_Z} \left( h^{-1}(\Theta_2 \setminus S_2) \right) = 0$.
\end{p2}

Let $\Phi_2 := h^{-1}(\Theta_2)$.
Since $h|_{\Phi_2} : \Phi_2 \to \Theta_2$ is a linear bijective function, the above statement can be interpreted as $\Theta_2 \setminus S_2$ being very small and thus the combination of G1 and G2 being a requirement that is met in the generic case.

\subsubsection{Proof of Proposition \ref{prop::genericity2}}

The proof idea for Proposition \ref{prop::genericity2} - similar as for Proposition \ref{prop::genericity1} - is that $h^{-1}(\Theta_2 \setminus S_2)$ is essentially contained in the union of the root sets of finitely many multivariate polynomials and hence is a Lebesgue null set.
To give a rigorous proof of Proposition \ref{prop::genericity2}, we first need to introduce some definitions which are very similar to those in Section \ref{sec::genericity1}, and establish a lemma.

%
%

Recall that $T_{(U_1,U_2)}(Q) = Q^2 - U_1 Q - U_2$ (see Section \ref{sec::id_D_zero}).

Within this section, given any $\phi \in \Phi_2$, let $W^\phi$ be some $K$-variate VAR process with parameters $h(\phi)$, and let $X^\phi$ denote the first $K_X$ and $Z^\phi$ denote the remaining $K - K_X$ components of $W^\phi$.
And also for this section, for any $\phi \in \Phi_2$ and $i \geq 0$, let $\Gamma_i(\phi) := \E[ W^\phi_t (W^\phi)_{t-i}^\top ]$.

Recall the definition of $M_1$ from Section \ref{sec::genres}. Here we explicitly consider $M_1$ as a function on $\Phi_2$. That is, for any $\phi \in \Phi_2$ let
\begin{align*}
 M_1(\phi) := \E \left[ W^\phi_t ( ( X^\phi_t )^\top , ( X^\phi_{t-1} )^\top ) \right] .
\end{align*}

\begin{Lemma}
\label{lem::distinct_roots}
Let $q_0(x_1, \ldots, x_m), \ldots, q_n(x_1, \ldots, x_m)$ be multivariate polynomials (elements of in $\R[x_1, \ldots, x_m]$).
Let 
\[ q(\alpha;x_1, \ldots, x_m) := q_0(x_1, \ldots, x_m) + q_1(x_1, \ldots, x_m) \alpha + \ldots + q_n(x_1, \ldots, x_m) \alpha^{n} , \]
i.e. a univariate polynomial in $\alpha$ (an element of $\R[\alpha]$) parametrized by $(x_1, \ldots, x_m)$.
If $q(\alpha;x_1, \ldots, x_m)$ has $n$ distinct roots for one $(x_1, \ldots, x_m) \in \R^m$, then
\[ \{ (x_1, \ldots, x_m) \in \R^m : q(\cdot;x_1, \ldots, x_m) \text{ does not have $n$ distinct roots} \} . \]
is a null set w.r.t. the $m$-dimensional Lebesgue measure on $\R^m$.
\end{Lemma}

\begin{proof}


Given two polynomials $r(\alpha), s(\alpha)$, let $S(r, s)$ denote their Sylvester matrix \citep{Dickenstein2010, WeissteinDiscr}.
Keep in mind that all entries of the Sylvester matrix $S(r, s)$ are either 0 or coincide with a coefficient of $r$ or $s$. Hence in particular, all entries of $S(r, s)$ are polynomials in the coefficients of $r$ and $s$.

Given a non-zero polynomial $p(\alpha) = p_0 + p_1 \alpha + \ldots + p_{\deg(p)} \alpha^{\deg(p)}$, let $\Delta(p)$ denote its discriminant, i.e.
\begin{align*}
 \Delta(p) := 
 p_{\deg(p)}^{2 \deg(p) - 2}  \prod_{i < j} (\alpha_i-\alpha_j)^2 ,
\end{align*}
where $\alpha_1, \ldots, \alpha_{\deg(p)}$ are the $\deg(p)$ complex roots of $p$, with potential multiplicities.

Keep in mind the following equation \citep{Dickenstein2010, WeissteinDiscr} that relates discriminant and Sylvester matrix: for all polynomials $p(\alpha)$ we have
\begin{align}
 (-1)^{\frac{1}{2} \deg(p) ( \deg(p) - 1)} p_{\deg(p)} \Delta(p) =  \det(S(p,p')), \label{eqn::discr_sylv}
\end{align}
where $p'(\alpha)$ is the derivative of $p(\alpha)$ w.r.t. $\alpha$.

Let
\[ s(x_1, \ldots, x_m) := \det(S(q(\cdot;x_1, \ldots, x_m),q'(\cdot;x_1, \ldots, x_m))) , \]
which is a multivariate polynomial in $(x_1, \ldots, x_m)$ based on the fact that the coefficients of $q(\cdot;x_1, \ldots, x_m)$ are multivariate polynomial in $(x_1, \ldots, x_m)$ and the determinant of the Sylvester matrix is a multivariate polynomial in the coefficients of $q(\cdot;x_1, \ldots, x_m)$. 

By assumption there is one $(x_1, \ldots, x_m) \in \R^m$ such that $q(\cdot;x_1, \ldots, x_m)$ has $n$ distinct roots. 
Based on equation (\ref{eqn::discr_sylv}) and the definition of $\Delta(q(\cdot;x_1, \ldots, x_m))$, this implies that for this $(x_1, \ldots, x_m)$, $s(x_1, \ldots, x_m) \neq 0$.
Based on Lemma \ref{lem::poly_roots}, $s(x_1, \ldots, x_m) \neq 0$ for all $(x_1, \ldots, x_m) \in \R^m \setminus L$, for some Lebesgue null set $L$.

Using equation (\ref{eqn::discr_sylv}) again, we know that $\Delta(q(\cdot;x_1, \ldots, x_m)) \neq 0$ for all $(x_1, \ldots, x_m) \in \R^m \setminus L$.
Hence $q(\cdot;x_1, \ldots, x_m)$ has $n$ distinct roots for all $(x_1, \ldots, x_m) \in \R^m \setminus L$.

\end{proof}

\begin{proof}[Proof of Proposition \ref{prop::genericity2}]

\stress{Prerequisties:}


Keep in mind that $\Phi_2 = h^{-1}(\Theta_2)$ and $\Psi_2 = h^{-1}(S_2)$.

Let 
\begin{align*}
\left( \begin{array}{cc}  \tilde{B}(\phi) &  \tilde{C}(\phi)  \\ 0 & \tilde{E}(\phi \end{array} \right) := \tilde{A}(\phi) := h_1(\phi) .
\end{align*}

Let $H$ denote the set of those $\phi \in \Phi_2$, for which $\tilde{C}(\phi)$ and $M_1(\phi)$ have full rank.
Let $H'$ denote the set of those $\phi \in \Phi_2$, for which $\tilde{A}(\phi)$ is such that there exists $U = (U_1,U_2)$ such that the equation
\begin{align}
 \left( U_1 , U_2 \right) \left( \begin{array}{cc}  \tilde{B}(\phi) &  \tilde{C}(\phi)  \\ \I & 0 \end{array} \right) = \left( \tilde{B}(\phi)^2 , \tilde{B}(\phi) \tilde{C}(\phi) + \tilde{C}(\phi) \tilde{E}(\phi) \right) , \label{eqn::U_eq}
\end{align}
or equivalently
\begin{align}
 \left( U_1 , U_2 \right) \left( \begin{array}{cc}   \tilde{C}(\phi) & \tilde{B}(\phi)  \\ 0 & \I \end{array} \right) = \left( \tilde{B}(\phi) \tilde{C}(\phi) + \tilde{C}(\phi) \tilde{E}(\phi) , \tilde{B}(\phi)^2 \right) \label{eqn::u}
\end{align}
is satisfied.

Keep in mind that $\Psi_2 = H \cap H'$.

Similar as in the proof of Proposition 1, it can be shown that 
\begin{align}
 \lambda_{2 K^2 - K_X K_Z}(\Phi_2 \setminus H) = 0. \label{eqn::H_0}
\end{align}
It remains to show the same for $H'$.

\stress{The case $K_Z = K_X$:}

Let $L_C$ denote the set of those $\phi \in \Phi_2$, for which $\tilde{C}(\phi)$ is not invertible.
As usual (see the proof of Proposition \ref{prop::genericity1}), Lemma \ref{lem::poly_roots} implies that $L_C$ has Lebesgue measure zero.

For all $\phi \in \R^{2 K^2 - K_X K_Z}$, define $U(\phi) = (U_1(\phi), U_2(\phi))$ as follows:

On $\R^{2 K^2 - K_X K_Z} \setminus L_C$ let
\begin{align}
 (U_1, U_2)  &:= \left( \tilde{B} \tilde{C} + \tilde{C} \tilde{E} , \tilde{B}^2 \right) \left( \begin{array}{cc}   \tilde{C}^{-1} & - \tilde{C}^{-1} \tilde{B}  \\ 0 & \I \end{array} \right) \label{eqn::uu} \\
 &= \left( \tilde{B} + \tilde{C} \tilde{E} \tilde{C}^{-1}  , - \tilde{B}^2 - \tilde{C} \tilde{E} \tilde{C}^{-1} \tilde{B} +  \tilde{B}^2 \right)  \\
 &= \left( \tilde{B} + \tilde{C} \tilde{E} \tilde{C}^{-1}  ,  - \tilde{C} \tilde{E} \tilde{C}^{-1} \tilde{B}  \right)  \\
 &= \left( \tilde{B} + \tilde{C} \tilde{E} \det( \tilde{C} )^{-1} \adj( \tilde{C} )  ,  - \tilde{C} \tilde{E} \det( \tilde{C} )^{-1} \adj( \tilde{C} ) \tilde{B}  \right)  \\
 &= \det( \tilde{C} )^{-1} \left( \det( \tilde{C} ) \tilde{B} + \tilde{C} \tilde{E}  \adj( \tilde{C} )  ,  - \tilde{C} \tilde{E} \adj( \tilde{C} ) \tilde{B}  \right)  ,
\end{align}
where, as usual, $\adj$ denotes the adjugate of a matrix.
Otherwise, on $L_C$, let $(U_1, U_2) := (0,0)$ (or anything else since this case does not matter).

On $\R^{2 K^2 - K_X K_Z} \setminus L_C$ we have
\begin{align*}
 &\det(T_{U}(\alpha)) \\
 &= \det( \alpha^2 \I - U_1 \alpha - U_2 ) \\
 &= \det( \tilde{C} )^{-K_X} \det\left( \det( \tilde{C} ) \alpha^2 \I - \alpha \left( \det( \tilde{C} ) \tilde{B} + \tilde{C} \tilde{E}  \adj( \tilde{C} ) \right) + \tilde{C} \tilde{E} \adj( \tilde{C} ) \tilde{B} \right) .
\end{align*}

Keep in mind that for any matrix $Q$, the determinant $\det(Q)$ as well as all entries of the adjugate $\adj(Q)$, are multivariate polynomials in the entries of $Q$.
(And obviously the entries of $\tilde{A}(\phi), \tilde{B}(\phi), \tilde{C}(\phi), \tilde{E}(\phi)$ are multivariate polynomials in $\phi$.)

Hence 
\begin{align}
\tilde{q}(\alpha,\phi) := \det\left( \det( \tilde{C}(\phi) ) \alpha^2 \I - \alpha \left( \det( \tilde{C}(\phi) ) \tilde{B}(\phi) + \tilde{C}(\phi) \tilde{E}(\phi)  \adj( \tilde{C}(\phi) ) \right) + \tilde{C}(\phi) \tilde{E}(\phi) \adj( \tilde{C}(\phi) ) \tilde{B}(\phi) \right)
\end{align}
is a multivariate polynomial in $(\alpha, \phi) \in \R \times \R^{2 K^2 - K_X K_Z}$. 
And in particular, considering $\phi$ as parameter vector, 
\[ q(\alpha;\phi) := \tilde{q}(\alpha,\phi) \]
is a univariate polynomial in $\alpha$, whose coefficients are all multivariate polynomials in $\phi$.
Note that $q(\alpha;\phi)$ has degree $2 K_X$ for all $\phi \in \R^{2 K^2 - K_X K_Z} \setminus L_C$, since it is up to a constant, which does not depend on $\alpha$, equal to $\det( \alpha^2 \I - U_1 \alpha - U_2 )$.

We want to apply Lemma \ref{lem::distinct_roots} to $q(\alpha;\phi)$.
For this purpose we need to show that there is a $\phi \in \R^{2 K^2 - K_X K_Z}$, such that $q(\alpha;\phi)$ has $2 K_X$ distinct roots.

Let $\phi$ be such that 
\begin{align}
\tilde{B}(\phi)&=\diag(1,3,5,\ldots,2 K_X - 1) \label{eqn::q_expl1} ,\\
\tilde{C}(\phi)&=\I ,\\
\tilde{E}(\phi) &= \diag(2,4,6,\ldots,2 K_X ) \label{eqn::q_expl3} .
\end{align}
For this $\phi$ we have
\begin{align*}
q(\alpha;\phi) &= \det(  \alpha^2 \I - \alpha ( \diag(1,3,5,\ldots,2 K_X - 1) + \diag(2,4,6,\ldots,2 K_X) ) \\
&\phantom{=} + \diag(1,3,5,\ldots,2 K_X - 1) \diag(2,4,6,\ldots,2 K_X )  ) \\
&= (\alpha^2  - (1+2) \alpha + 1 \cdot 2) (\alpha^2  - (3+4) \alpha + 3 \cdot 4) \cdot \ldots \\
&\phantom{=} \cdot (\alpha^2  - (2 K_X - 1+2 K_X ) \alpha + (2 K_X - 1) 2 K_X ) \\
&= (\alpha - 1)(\alpha - 2) (\alpha - 3)(\alpha - 4) \cdot \ldots \cdot (\alpha - (2 K_X - 1))(\alpha - 2 K_X + 2)
\end{align*}
hence $q(\alpha;\phi)$ has the distinct roots $1,2,\ldots,2 K_X$.

Now Lemma \ref{lem::distinct_roots} implies that $q(\alpha;\phi)$ has $2 K_X$ distinct roots for all $\phi \in \R^{2 K^2 - K_X K_Z} \setminus L$, for some $L$ with $\lambda_{2 K^2 - K_X K_Z}(L) = 0$.

Keep in mind that
\[ \det(T_{U(\phi)}(\alpha)) = \det(C)^{-1} q(\alpha;\phi) \]
for all $\phi \in \R^{2 K^2 - K_X K_Z} \setminus L_C$.
Hence $\det(T_{U(\phi)}(\alpha))$ has $2 K_X$ distinct roots for all $\phi \in \R^{K^2 + K} \setminus (L \cup L_C)$.  
Moreover, for all $\phi \in \R^{2 K^2 - K_X K_Z} \setminus (L \cup L_C)$, $U(\phi)$ satisfies equation (\ref{eqn::U_eq}) by its definition.
This implies $(L \cup L_C)^C \subset H'$ and in particular $(H')^C \subset L \cup L_C$, where $(\cdot)^C$ denotes the complement of a set, as usual.
Hence $\lambda_{2 K^2 - K_X K_Z}( (H')^C ) = 0$.

Using the fact that $\Psi_2^C = ( H \cap H')^C = H^C \cup ( H')^C$ and equation (\ref{eqn::H_0}) we can calculate
\begin{align*}
\lambda_{2 K^2 - K_X K_Z} \left( h^{-1}(\Theta_2 \setminus S_2) \right) &= \lambda_{2 K^2 - K_X K_Z}( \Phi_2 \setminus \Psi_2 ) \\
&= \lambda_{2 K^2 - K_X K_Z}( \Phi_2 \cap \Psi_2^C ) \\
&= \lambda_{2 K^2 - K_X K_Z}( \Phi_2 \cap ( H^C \cup ( H')^C ) ) \\
&= \lambda_{2 K^2 - K_X K_Z}( ( \Phi_2 \cap H^C ) \cup ( \Phi_2 \cap ( H')^C ) ) \\
&\leq \lambda_{2 K^2 - K_X K_Z}(  \Phi_2 \cap H^C  ) + \lambda_{2 K^2 - K_X K_Z}( \Phi_2 \cap ( H')^C )  \\
&= 0 .
\end{align*}

\stress{Second, the case $K_Z < K_X$:}

This case works similarly as the case $K_Z = K_X$.

Let $\I_{m}$ denote the $m \times m$ identity matrix and $0_{m \times n}$ the $m \times n$ zero matrix. 
For the sake of a simple notation, here we suppress the dependence on $\phi$.
Let 
\[d := \diag(2,4,6, \ldots, 2 (K_X - K_Z) )\]
and
\begin{align*}
\hat{B} &:= \tilde{B} ,\\
\hat{C} &:= \left( \left. \begin{array}{c}   \I_{K_X - K_Z}   \\ 0_{K_Z \times (K_X - K_Z)}  \end{array} \right| \tilde{C} \right) . \\
\hat{E} &:= \left( \begin{array}{cc}   d & 0_{(K_X - K_Z) \times K_Z}   \\ 0_{K_Z \times (K_X - K_Z)} & \tilde{E}  \end{array}\right) . \\
\end{align*}
Note that $\hat{B},\hat{C},\hat{E}$ all have dimension $K_X \times K_X$.

Now the argument is similar as for the case $K_Z = K_X$, except that we replace $\tilde{B},\tilde{C},\tilde{E}$ by $\hat{B},\hat{C},\hat{E}$.

Let us briefly comment on two points.

First, similar as for the case $K_Z = K_X$, whenever $\hat{C}$ is invertible, we define
\begin{align}
 (U_1, U_2)  &:= \left( \hat{B} \hat{C} + \hat{C} \hat{E} , \hat{B}^2 \right) \left( \begin{array}{cc}   \hat{C}^{-1} & - \hat{C}^{-1} \hat{B}  \\ 0 & \I \end{array} \right) .
\end{align}
(The argument for $\tilde{C}$ to almost always have full rank and thus $\hat{C}$ almost always being invertible carries over from the case $K_Z = K_X$.)
This implies that $(U_1, U_2)$ satisfies
\begin{align*}
 &\left( U_1 , U_2 \right) \left( \begin{array}{c|c|c}  \I_{K_X - K_Z} &    \tilde{C} & \tilde{B}  \\ 0_{K \times (K_X - K_Z)} & 0 & \I \end{array} \right) \\
 &= \left( U_1 , U_2 \right) \left( \begin{array}{cc}   \hat{C} & \hat{B}  \\ 0 & \I \end{array} \right) \\
 &= \left( \hat{B} \hat{C} + \hat{C} \hat{E} , \hat{B}^2 \right) \\
 &= \left( \tilde{B} \left( \left. \begin{array}{c}   \I_{K_X - K_Z}   \\ 0_{K_Z \times (K_X - K_Z)}  \end{array} \right| \tilde{C} \right) + \left( \left. \begin{array}{c}   \I_{K_X - K_Z}   \\ 0_{K_Z \times (K_X - K_Z)}  \end{array} \right| \tilde{C} \right) \left( \begin{array}{cc}   d & 0_{(K_X - K_Z) \times K_Z}   \\ 0_{K_Z \times (K_X - K_Z)} & \tilde{E}  \end{array}\right) , \tilde{B}^2 \right) \\
 &= \left( \left( \left. \tilde{B}\left( \begin{array}{c}    \I_{K_X - K_Z}   \\ 0_{K_Z \times (K_X - K_Z)}  \end{array} \right) \right|   \tilde{B} \tilde{C} \right) + \left( \left.  \begin{array}{c}    d   \\ 0_{K_Z \times (K_X - K_Z)}  \end{array} \right|   \tilde{C} \tilde{E} \right) , \tilde{B}^2 \right) \\
 &= \left(  \left( \left. \tilde{B}\left( \begin{array}{c}    \I_{K_X - K_Z}   \\ 0_{K_Z \times (K_X - K_Z)}  \end{array} \right) + \left( \begin{array}{c}    d   \\ 0_{K_Z \times (K_X - K_Z)}  \end{array} \right) \right|   \tilde{B} \tilde{C} + \tilde{C} \tilde{E} \right) , \tilde{B}^2 \right) \\
  &= \left(    \tilde{B}\left( \begin{array}{c}    \I_{K_X - K_Z}   \\ 0_{K_Z \times (K_X - K_Z)}  \end{array} \right) + \left( \begin{array}{c}    d   \\ 0_{K_Z \times (K_X - K_Z)}  \end{array} \right) , \tilde{B} \tilde{C} + \tilde{C} \tilde{E}  , \tilde{B}^2 \right) ,
 \end{align*}
whenever $\hat{C}$ is invertible.
Hence $(U_1, U_2)$ also satisfies equation (\ref{eqn::u}), whenever $\hat{C}$ is invertible.

Second, keep in mind how we, in the case $K_Z = K_X$, constructed the sample $\phi$ such that $q(\alpha;\phi)$ had $2 K_X$ distinct roots. We used equations (\ref{eqn::q_expl1}) to (\ref{eqn::q_expl3}).
Note that the way we constructed $\hat{B},\hat{C},\hat{E}$ here, there has to be a $\phi$ such that these equations hold true for $\hat{B},\hat{C},\hat{E}$ instead of $\tilde{B},\tilde{C},\tilde{E}$.
Now with the analogous calculation as in the case $K_Z = K_X$, it follows that for this $\phi$, $q(\alpha;\phi)$ has $2 K_X$ distinct roots.


%
%
%
%
%

\end{proof}

\section{Algorithm \ref{alg::vem_algo} in Detail}
\label{sec::vem_algo_details}

Here we describe Algorithm \ref{alg::vem_algo} introduced in Section \ref{sec::vem_algo} in detail.
The approach is similar to the one in \citep{Oh05avariational}.

\subsection{The Likelihood and Its Approximation}

Here we assume the general model specified in Section \ref{sec::model} and additionally that for each $i=1,...,K$ the density $p_{n_i}$ of the noise term $N_t^i$ is a mixture of $p_i$ Gaussians, i.e., $p_{n_i} = \sum_{c=1}^{p_i} \pi_{i,c}\mathcal{N}(n_i|\mu_{i,c}, \sigma_{i,c}^2)$, where $\pi_{i,c} \geq 0$, $\sum_{c=1}^{p_i}\pi_{i,c} = 1$.
In what follows, we denote the values of the sample $X_{1:L}$ by $x_{1:L}$, the values of the hidden variables $Z_{1:L}$ by $z_{1:L}$, and the values of the vectors $V^X_{1:L}, V^Z_{1:L}$ that select the mixture component of $N^X_{1:L}, N^Z_{1:L}$ by $v^X_{1:L}, v^Z_{1:L}$.

We can write down the complete-data likelihood as
\begin{flalign}
p(x_{1:L},z_{1:L},v_{1:L}^X,v_{1:L}^Z) =& \left[\prod_{l=1}^L p(v_l^X) p(v^Z_l)\right]p(z_1|v^Z_1)\left[\prod_{l=2}^L p(z_l|z_{l-1},x_{l-1},v^Z_l)\right]p(x_1|v^X_1)\nonumber\\
&\left[\prod_{l=2}^L p(x_l|x_{l-1},z_{l-1},v^X_{l})\right],
\end{flalign}
where
\begin{flalign}
p(v^X_l) =& \prod_{i=1}^{K_X} p(v^X_{l,i}) = \prod_{i=1}^{K_X} \pi_{i+K_Z,v^X_{l,i}},\\
p(v^Z_l) =& \prod_{i=1}^{K_Z} p(v^Z_{l,i}) = \prod_{i=1}^{K_Z} \pi_{i,v^Z_{l,i}},
\end{flalign}
\begin{flalign}
p(x_l|x_{l-1},z_{l-1},v^X_{l}) =& \mathcal{N}(x_l|Bx_{l-1}+Cz_{l-1}+\mu_{v^X_l},\Sigma_{v^X_l}),\\
p(z_l|z_{l-1},x_{l-1},v^Z_{l}) =& \mathcal{N}(z_l|Ez_{l-1}+Dx_{l-1}+\mu_{v^Z_l},\Sigma_{v^Z_l}),
\end{flalign}
\begin{flalign}
\mu_{v^X_{l}} =& (\mu_{K_Z+1,v^X_{l,1}},...,\mu_{K,v^X_{l,K_X}})^\intercal,
\mu_{v^Z_{l}} = (\mu_{1,v^Z_{l,1}},...,\mu_{K_Z,v^Z_{l,K_Z}})^\intercal,\\
\Sigma_{v^X_{l}} =& \text{diag}(\sigma^2_{K_Z+1,v^X_{l,1}},...,\sigma^2_{K,v^X_{l,K_X}}), ~
\Sigma_{v^Z_{l}} = \text{diag}(\sigma^2_{1,v^Z_{l,1}},...,\sigma^2_{K_Z,v^Z_{l,K_Z}}).
\end{flalign}

Instead of maximizing the marginal likelihood $p(x_{1:L})$, we maximize the EM lower bound of $p(x_{1:L})$, which leads to the EM algorithm. In the E-step, the posterior of the hidden variables $p(z_{1:L},v^X_{1:L},v^Z_{1:L}|x_{1:L})$ is intractable because the number of Gaussian mixtures grows exponentially with the length of the time series. Thus, approximations must be made to make the problem tractable. We use a factorized approximate posterior 
\[ p(z_{1:L},v^X_{1:L},v^Z_{1:L}|x_{1:L})\approx q(z_{1:L}|x_{1:L})q(v^X_{1:L},v^Z_{1:L}|x_{1:L})
\]
 to approximate the true posterior based on the mean-field assumption. Then the variational EM lower bound can be written as
\begin{flalign}
 \mathcal{L} =& \sum_{v^X_{1:L},v^Z_{1:L}}q(v^X_{1:L},v^Z_{1:L}|x_{1:L})\int dz_{1:L}~ q(z_{1:L}|x_{1:L})\ln{p(x_{1:L},z_{1:L},v^X_{1:L},v^Z_{1:L})}\nonumber\\
 &-\sum_{v^X_{1:L},v^Z_{1:L}}q(v^X_{1:L},v^Z_{1:L}|x_{1:L})\ln{q(v^X_{1:L},v^Z_{1:L}|x_{1:L})} - \int dz_{1:L}~ q(z_{1:L}|x_{1:L})\ln{q(z_{1:L}|x_{1:L})}\nonumber\\
 =& \sum_{l=1}^L\sum_{v^X_{l}}q(v^X_{l}|x_{1:L})\ln{p(v^X_{l})}+\sum_{l=1}^L\sum_{v^Z_{l}}q(v^Z_{l}|x_{1:L})\ln{p(v^Z_{l})}+\sum_{v^Z_{1}}q(v^Z_{1}|x_{1:L})\int dz_{1}~q(z_{1}|x_{1:L})\ln{p(z_{1}|v^Z_1)}\nonumber\\
 &+\sum_{l=2}^L\sum_{v^Z_{l}}q(v^Z_{l}|x_{1:L})\int dz_{l}dz_{l-1}~q(z_{l},z_{l-1}|x_{1:L})\ln{p(z_{l}|z_{l-1},x_{l-1},v^Z_{l})}+\sum_{v^X_{1}}q(v^X_{1}|x_{1:L})\ln{p(x_{1}|v^X_1)}\nonumber\\
&+ \sum_{l=2}^L\sum_{v^X_{l}}q(v^X_{l}|x_{1:L})\int dz_{l-1} q(z_{l-1}|x_{1:L})\ln{p(x_l|x_{l-1},z_{l-1},v^X_{l})}\nonumber\\
&-\sum_{v^X_{1:L},v^Z_{1:L}}q(v^X_{1:L},v^Z_{1:L}|x_{1:L})\ln{q(v^X_{1:L},v^Z_{1:L}|x_{1:L})} - \int dz_{1:L}~ q(z_{1:L}|x_{1:L})\ln{q(z_{1:L}|x_{1:L})}
\end{flalign}

\subsection{The Algorithm}

In the variational E step, $q(z_{1:L}|x_{1:L})$ and $q(v^X_{1:L},v^Z_{1:L}|x_{1:L})$ are updated alternately by maximizing the variational lower bound. The update rules are as follows
\begin{flalign}
q(v^X_{1:L},v^Z_{1:L}|x_{1:L}) \leftarrow \frac{1}{c_{v^X v^Z}}\exp{\left<\ln {p(x_{1:L},z_{1:L},v^X_{1:L},v^Z_{1:L})}\right>_{q(z_{1:L}|x_{1:L})}},\label{qz1}\\
q(z_{1:L}|x_{1:L}) \leftarrow \frac{1}{c_{z}}\exp{\left<\ln {p(x_{1:L},z_{1:L},v^X_{1:L},v^Z_{1:L})}\right>_{q(v^X_{1:L},v^Z_{1:L}|x_{1:L})}}\label{qe1}
\end{flalign}
In (\ref{qz1}), the expectation of the log-likelihood with respect to $q(z_{1:L}|x_{1:L})$ is calculated as
\begin{flalign}
&\left<\ln {p(x_{1:L},z_{1:L},v^X_{1:L},v^Z_{1:L})}\right>_{q(z_{1:L}|x_{1:L})} \\
&=\sum_{l=1}^L\sum_{i=1}^{K_X}\ln{p(v^X_{l,i})}+\sum_{l=1}^L\sum_{i=1}^K\ln{p(v^Z_{l,i})}\nonumber\\
&-\frac{1}{2}\sum_{i=1}^{v_Z} \left(\frac{\left<(z_{1,i}-\mu_{i,v^Z_{1,i}})^2\right>_{q(z_{1,i}|x_{1:L})}}{{\sigma^2_{i,v^Z_{1,i}}}}+2\ln{\sigma_{i,v^Z_{1,i}}}\right)\nonumber\\
&-\frac{1}{2}\sum_{l=2}^L\sum_{i=1}^{K_Z} \left(\frac{\left<\left(z_{l,i}-(Ez_{l-1})_i-(Dx_{l-1})_i-\mu_{i,v^Z_{l,i}}\right)^2\right>_{q(z_l,z_{l-1}|x_{1:L})}}{{\sigma^2_{i,v^Z_{l,i}}}}+2\ln{\sigma_{i,v^Z_{l,i}}}\right)\nonumber\\
&-\frac{1}{2}\sum_{i=1}^{K_X}\left(\frac{(x_{1,i}-\mu_{i+K_Z,v^X_{1,i}})^2}{\sigma^2_{i+K_Z,v^X_{1,i}}}+2\ln{\sigma_{i+K_Z,v^X_{1,i}}}\right)+\const\nonumber\\
&-\frac{1}{2}\sum_{l=2}^L\sum_{i=1}^{K_X} \left(\frac{\left<\left(x_{l,i}-(Cz_{l-1})_i-(Bx_{l-1})_i-\mu_{i+K_Z,v^X_{l,i}}\right)^2\right>_{q(z_{l-1}|x_{1:L})}}{{\sigma^2_{i+K_Z,v^X_{l,i}}}}+2\ln{\sigma_{i+K_Z,v^X_{l,i}}}\right).\nonumber\\
\end{flalign}
It can be seen that $q(v^X_{1:L},v^Z_{1:L}|x_{1:L})$ further factorizes as $\left[\prod_l\prod_i q(v^X_{l,i})\right]\left[\prod_l\prod_i q(v^Z_{l,i})\right]$, which means the posterior $q(v^X_{1:L},v^Z_{1:L}|x_{1:L})$ can be calculated separately for each channel. The computational complexity is linear in the time series length, the number of time series channels, and the number of Gaussian mixtures in each channel.

(\ref{qe1}) can be further expressed as
\begin{flalign}
&\left<\ln {p(x_{1:L},z_{1:L},v^X_{1:L},v^Z_{1:L})}\right>_{q(v^X_{1:L},v^Z_{1:L}|x_{1:L})} \\
&=-\frac{1}{2}\sum_{i=1}^{K_Z} z_{1,i}^2\left(\sum_{v^Z_{1,i}}\frac{q(v^Z_{1,i})}{\sigma^2_{i,v^Z_{1,i}}}\right)+\sum_{i=1}^{K_Z} z_{1,i}\left(\sum_{v^Z_{1,i}}\frac{q(v^Z_{1,i})\mu_{i,v^Z_{1,i}}}{\sigma^2_{i,v^Z_{1,i}}}\right)\nonumber\\
&-\frac{1}{2}\sum_{l=2}^L\sum_{i=1}^{K_Z} \left(z_{l,i}-(E z_{l-1}\right)_i)^2\left(\sum_{v^Z_{l,i}}\frac{q(v^Z_{l,i})}{\sigma^2_{i,v^Z_{l,i}}}\right)\nonumber\\
&+\sum_{l=2}^L\sum_{i=1}^{K_Z} \left(z_{l,i}-(Ez_{l-1}\right)_i)\left(\sum_{v^Z_{l,i}}\frac{q(v^Z_{l,i})\left((Dx_{l-1})_i+\mu_{i,v^Z_{l,i}}\right)}{\sigma^2_{i,v^Z_{l,i}}}\right)\nonumber\\
&-\frac{1}{2}\sum_{l=2}^L\sum_{i=1}^{K_X} \left(x_{l,i}-(Bx_{l-1})_i-(Cz_{l-1})_i\right)^2\left(\sum_{v^X_{l,i}}\frac{q(v^X_{l,i})}{\sigma^2_{i+K_Z,v^X_{l,i}}}\right)\nonumber\\
&+\sum_{l=2}^L\sum_{i=1}^{K_X} \left(x_{l,i}-(Bx_{l-1})_i-(Cz_{l-1})_i\right)\left(\sum_{v^X_{l,i}}\frac{q(v^X_{l,i})\mu_{i+K_Z,v^X_{l,i}}}{\sigma^2_{i+K_Z,v^X_{l,i}}}\right)+\const,
\end{flalign}
which has the form of the joint log-likelihood function of 
a time-varying linear dynamical system (LDS). The marginal posteriors $p(z_l|x_{1:L})$ and $p(z_l,z_{l-1}|x_{1:L})$ can be obtained by Kalman filter and smoothing algorithms.

In the M-step, we maximize the variational lower bound with respect to the parameters given the marginal posterior distributions from the E-step. The update rules for the parameters are given as follows
\begin{flalign}
\pi_{i,c} = 
\left\{
\begin{array}{c l}      
    \frac{1}{L}\sum_{l=1}^L q(v^Z_{l,i}=c|x_{1:L}), & i=1,...,K_Z,\\
    \frac{1}{L}\sum_{l=1}^L q(v^X_{l,i-K_Z}=c|x_{1:L}), & i=K_Z+1,...,K,
\end{array}\right.
\end{flalign}

\begin{flalign}
\mu_{i,c} = 
\left\{
\begin{array}{c l}      
    &\frac{q(v^Z_{1,i}=c|x_{1:L})\left(\left<z_{1,i}\right>_{q(z_{1,i}|x_{1:L})}\right)+\sum_{l=2}^L q(v^Z_{l,i}=c|x_{1:L})\left(\left<z_{l,i}\right>_{q(z_{l,i}|x_{1:L})}-\left(E\left<z_{l-1}\right>_{q(z_{l-1}|x_{1:L})}\right)_i-\left(Dx_{l-1}\right)_i\right)}{\sum_{l=1}^L q(v^Z_{l,i}=c|x_{1:L})} , \\
     &i=1,...,K_Z,\\
    &\frac{q(v^X_{1,i-K_Z}=c|x_{1:L})x_{1,i-K_Z}+\sum_{l=2}^L q(v^X_{l,i-K_Z}=c|x_{1:L})\left(x_{l,i-K_Z}-\left(C\left<z_{l-1}\right>_{q(z_{l-1}|x_{1:L})}\right)_{i-K_Z}-\left(Bx_{l-1}\right)_i\right)}{\sum_{l=1}^L q(v^X_{l,i-K_Z}=c|x_{1:L})} , \\
    & i=K_Z+1,...,K,
\end{array}\right.
\end{flalign}

\begin{flalign}
\sigma^2_{i,c} = 
\left\{
\begin{array}{cl}      
    &\frac{q(v^Z_{1,i}=c|x_{1:L})\left(\left<z^2_{1,i}-2\mu_{i,c}z_{1,i}\right>_{q(z_{1,i}|x_{1:L})}\right)+\sum_{l=2}^L q(v^Z_{l,i}=c|x_{1:L})\left\{\left[z_{l,i}-\left(Ez_{l-1}\right)_i-\left(Dx_{l-1}\right)_i\right]^2_{q(z_{l},z_{l-1}|x_{1:L})}\right.}{\sum_{l=1}^L q(v^Z_{l,i}=c|x_{1:L})}   \\
     &\frac{\left. -2\mu_{i,c}\left[\left<z_{l,i}\right>_{q(z_{l,i}|x_{1:L})}-\left(E\left<z_{l-1}\right>_{q(z_{l-1}|x_{1:L})}\right)_i-\left(D x_{l-1}\right)_i\right]\right\}}{\sum_{l=1}^L q(v^Z_{l,i}=c|x_{1:L})} +\mu^2_{i,c},\\
     & i=1,...,K_Z, \\
    &\frac{q(v^X_{1,i-K_Z}=c|x_{1:L})\left(x^2_{1,i-K_Z}-2\mu_{i,c}x_{1,i-K_Z}\right)+\sum_{l=2}^L q(v^X_{l,i-K_Z}=c|x_{1:L})\left\{\left[x_{l,i-K_Z}-\left(C z_{l-1}\right)_{i-K_Z}-\left(B x_{l-1}\right)_{i-K_Z}\right]^2_{q(z_{l-1}|x_{1:L})}\right.}{\sum_{l=1}^L q(v^X_{l,i-K_Z}=c|x_{1:L})}   \\
     &\frac{\left. -2\mu_{i,c}\left[x_{l,i-K_Z}-\left(C\left<z_{l-1}\right>_{q(z_{l-1}|x_{1:L})}\right)_{i-K_Z}-\left(Bx_{l-1}\right)_{i-m}\right]\right\}}{\sum_{l=1}^L q(v^X_{l,i-K_Z}=c|x_{1:L})}+\mu^2_{i,c}, \\
     &i=K_Z+1,...,K,
\end{array}\right.
\end{flalign}

\begin{flalign}
E_i =& {\left(\sum_{l=2}^L\sum_{v^Z_{l,i}}\frac{q(v^Z_{l,i}|x_{1:L})}{\sigma^2_{i,v^Z_{l,i}}}\left<z_{l-1}z_{l-1}^\top\right>_{q(z_{l-1}|x_{1:L})}\right)}^{-1}\left(\sum_{l=2}^L\sum_{v^Z_{l,i}}\frac{q(v^Z_{l,i}|x_{1:L})}{\sigma^2_{i,v^Z_{l,i}}}\left(\left<z_{l-1}z_{l,i}\right>_{q(z_{l},z_{l-1}|x_{1:L})}-\nonumber\right.\right.\\
&\left.\left.\left< z_{l-1}\right>_{q( z_{l-1}|x_{1:L})}(Dx_{l-1})_i-\left<z_{l-1}\right>_{q(z_{l-1}|x_{1:L})}\mu_{i,v^Z_{l,i}}\right) \vphantom{\sum_{l=2}^L\sum_{v^Z_{l,i}}} \right),
\end{flalign}

\begin{flalign}
D_i =& {\left(\sum_{l=2}^L\sum_{v^Z_{l,i}}\frac{q(v^Z_{l,i}|x_{1:L})}{\sigma^2_{i,v^Z_{l,i}}}x_{l-1}x_{l-1}^\top\right)}^{-1}\left(\sum_{l=2}^L\sum_{v^Z_{l,i}}\frac{q(v^Z_{l,i}|x_{1:L})}{\sigma^2_{i,v^Z_{l,i}}}x_{l-1}\left(\left<z_{l,i}\right>_{q(z_{l,i}|x_{1:L})}-\nonumber\right.\right.\\
&\left.\left.(E\left<z_{l-1}\right>_{q(z_{l-1}|x_{1:L})})_i-\mu_{i,v^Z_{l,i}}\right) \vphantom{\sum_{l=2}^L\sum_{v^Z_{l,i}}} \right),
\end{flalign}

\begin{flalign}
C_i =& {\left(\sum_{l=2}^T\sum_{v^X_{l,i}}\frac{q(v^X_{l,i}|x_{1:L})}{\sigma^2_{i+K_Z,v^X_{l,i}}}\left<z_{l-1}z_{l-1}^\top\right>_{q(z_{l-1}|x_{1:L})}\right)}^{-1}\left(\sum_{l=2}^L\sum_{v^X_{l,i}}\frac{q(v^X_{l,i}|x_{1:L})}{\sigma^2_{i+K_Z,v^X_{l,i}}}\left<z_{l-1}\right>_{q(z_{l-1}|x_{1:L})}\left(x_{l,i}-\nonumber\right.\right.\\
&\left.\left. (Bx_{l-1})_i-\mu_{i+K_Z,v^X_{l,i}}\right) \vphantom{\sum_{l=2}^L\sum_{v^Z_{l,i}}} \right),
\end{flalign}

\begin{flalign}
B_i =& {\left(\sum_{l=2}^L\sum_{v^X_{l,i}}\frac{q(v^X_{l,i}|x_{1:L})}{\sigma^2_{i+K_Z,v^X_{l,i}}}x_{l-1}x_{l-1}^\top\right)}^{-1}\left(\sum_{l=2}^L\sum_{v^X_{l,i}}\frac{q(v^X_{l,i}|x_{1:L})}{\sigma^2_{i+K_Z,v^X_{l,i}}}x_{l-1}\left(x_{l,i}-\nonumber\right.\right.\\
&\left.\left. \left(C\left<z_{l-1}\right>_{q(z_{l-1}|x_{1:L})}\right)_i-\mu_{i+K_Z,v^X_{l,i}}\right) \vphantom{\sum_{l=2}^L\sum_{v^Z_{l,i}}} \right),
\end{flalign}
where $E_i$, $D_i$, $C_i$, and $B_i$ denote the i-th row of $E$, $D$, $C$, and $B$ respectively.

\newpage

\bibliographystyle{icml2015}
\bibliography{../../Include/univbib,kun-bib}

\end{document}